\documentclass{article}

\usepackage{microtype}
\usepackage{graphicx}
\usepackage{subfigure}
\usepackage{booktabs} 

\usepackage{hyperref}

\usepackage[accepted]{icml2024}

\usepackage{amsmath}
\usepackage{amssymb}
\usepackage{mathtools}
\usepackage{amsthm}

\usepackage[capitalize,noabbrev]{cleveref}

\theoremstyle{plain}
\newtheorem{theorem}{Theorem}[section]

\newtheorem{lemma}[theorem]{Lemma}

\theoremstyle{definition}

\newtheorem{assumption}[theorem]{Assumption}
\theoremstyle{remark}
\newtheorem{remark}[theorem]{Remark}
\newtheorem{example}[theorem]{Example}

\usepackage[textsize=tiny]{todonotes}

\usepackage{threeparttable}
\usepackage{booktabs} 
\usepackage{rotating}
\usepackage{tablefootnote}
\usepackage{multirow}
\usepackage{hhline}
\usepackage{color, colortbl}

\newcommand{\ee}{\mathbb{E}}
\newcommand{\prob}{\mathbb{P}}
\newcommand{\norm}[1]{\left\lVert#1\right\rVert}
\DeclarePairedDelimiter\ceil{\lceil}{\rceil}

\icmltitlerunning{Efficient Exploration in Average-Reward Constrained RL: Achieving Near-Optimal Regret With Posterior Sampling}

\begin{document}

\twocolumn[
\icmltitle{Efficient Exploration in Average-Reward Constrained Reinforcement Learning: \\ Achieving Near-Optimal Regret With Posterior Sampling}



\icmlsetsymbol{equal}{*}

\begin{icmlauthorlist}
\icmlauthor{Danil Provodin}{tue,jads}
\icmlauthor{Maurits Kaptein}{tue}
\icmlauthor{Mykola Pechenizkiy}{tue,jyv}
\end{icmlauthorlist}

\icmlaffiliation{tue}{Eindhoven University of Technology, Eindhoven, The Netherlands}
\icmlaffiliation{jyv}{University of Jyväskylä, Jyväskylä, Finland}
\icmlaffiliation{jads}{Jheronimus Academy of Data Science, ‘s-Hertogenbosch, The Netherlands}

\icmlcorrespondingauthor{Danil Provodin}{d.provodin@tue.nl}

\icmlkeywords{Machine Learning, ICML}

\vskip 0.3in
]



\printAffiliationsAndNotice{}  

\begin{abstract}
We present a new algorithm based on posterior sampling for learning in Constrained Markov Decision Processes (CMDP) in the infinite-horizon undiscounted setting. The algorithm achieves near-optimal regret bounds while being advantageous empirically compared to the existing algorithms. Our main theoretical result is a Bayesian regret bound for each cost component of $\tilde{O} (DS\sqrt{AT})$ for any communicating CMDP with $S$ states, $A$ actions, and diameter $D$. This regret bound matches the lower bound in order of time horizon $T$ and is the best-known regret bound for communicating CMDPs achieved by a computationally tractable algorithm. Empirical results show that our posterior sampling algorithm outperforms the existing algorithms for constrained reinforcement learning.
\end{abstract}

\section{Introduction}

Reinforcement learning (RL) refers to the problem of learning by trial and error in sequential decision-making systems based on the scalar signal aiming to minimize the total cost accumulated over time.
In many situations, however, the desired properties of the agent behavior are better described using constraints, as a single objective might not suffice to explain the real-life setting. For example, a robot should not only fulfill its task but should also control its wear and tear by limiting the torque exerted on its motors \citep{tessler2018reward}; for telecommunication networks, it is necessary that the average end-to-end delay be limited, especially for voice traffic, while maximizing the throughput of the system \citep{Altman99constrainedmarkov}; and autonomous driving vehicles should reach the destination in a time and fuel-efficient manner while obeying traffic rules \cite{Le2019BatchPL}. A natural approach for handling such cases is specifying the problem using multiple objectives, where one objective is optimized subject to constraints on the others.

A typical way of formulating the constrained RL problem is a Constrained Markov Decision Process (CMDP) \citep{Altman99constrainedmarkov}, which proceeds in discrete time steps. At each time step, the system occupies a \textit{state}, and the decision maker chooses an \textit{action} from the set of allowable actions. As a result of choosing the action, the decision maker receives a (possibly stochastic) vector of \textit{costs}, and the system then transitions to the next state according to a fixed \textit{state transition distribution}. In the reinforcement learning problem, the underlying state transition distributions and/or cost distributions are unknown and need to be learned from observations while aiming to minimize the total cost. 

Learning in CMDPs has been a recurrent topic in the reinforcement learning literature, with numerous works addressing this challenge in episodic and discounted settings (see, e.g., \citet{Efroni_2020_CMDP, NEURIPS2020_Brantley, qiu_2021_cmdp_posterior, Liu_2021, kalagarla2023safe}). We consider the reinforcement learning problem in a more general \textit{infinite-horizon average reward} setting. When decisions are made frequently so that the discount rate is very close to 1, the decision-makers may prefer to compare policies on the basis of their expected infinite-horizon average reward instead of the expected total discounted reward, and the objective becomes to achieve optimal long-term average performance under constraints. This criterion is especially relevant for inventory systems with frequent restocking decisions or queueing control theory, particularly when applied to controlling computer systems \citep{Puterman_mdp}. 

\begin{table*}
\caption{ Summary of work on provably efficient constrained RL in the infinite-horizon average reward setting. $S$ and $A$ represent the number of states and actions, $m$ is the number of constraints, $T$ is the total horizon,  $T_M$ is the mixing time, $D$ is the diameter of CMDP, $p$ represents transitions, and $sp(p)$ is the span of CMDP (defined in Section \ref{sec:problem_formulation}). $\tilde{O}$ hides logarithmic factors. The ``Required knowledge'' column denotes the information an algorithm requires as an input. The ``Computation'' column roughly denotes the time complexity with ``Efficient'' meaning an algorithm is designed to solve a problem using minimal resources, ``Inefficient'' -- an algorithm consumes more time than necessary, and ``Intractable'' -- an algorithm for which no known polynomial-time solution exists.}
\resizebox{\textwidth}{!}{%
  \centering
  \begin{threeparttable}
  \begin{tabular}{lccccccc}
    \hline
    &&& Constraint && Required \\ 
    & \multirow{-2}{*}{Algorithm} & \multirow{-2}{*}{Main Regret} & violation & \multirow{-2}{*}{CMDP class} & knowledge & \multirow{-2}{*}{Computation} \\ 

    \hline

    & \textsc{C-UCRL} &&&& safe policy $\pi$ &\\
      
    & \cite{pmlr-v120-zheng20a} &  \multirow{-2}{*}{$\tilde{O}(mSAT^{3/4} )$} & \multirow{-2}{*}{0} & \multirow{-2}{*}{ergodic} & and $p$ & \multirow{-2}{*}{efficient} \\ 
   \cline{2-8}
    
     & \textsc{UCRL-CMDP} &  &&&&& \\
     
    & \cite{Singh_CMDP_2020} & \multirow{-2}{*}{$\tilde{O}(T_M\sqrt{SA}T^{2/3} )$} & \multirow{-2}{*}{$\tilde{O}(T_M\sqrt{SA}T^{2/3} )$} & \multirow{-2}{*}{ergodic} & \multirow{-2}{*}{$T$} & \multirow{-2}{*}{inefficient} \\
   \cline{2-8}

   & Alg. 3 &&& weakly &  &\\
    & \cite{chen_2022_optimisQlearn} & \multirow{-2}{*}{$\tilde{O}(sp(p)(S^2 A T^2)^{1/3}  )$} & \multirow{-2}{*}{$\tilde{O}( sp(p)(S^2 A T^2)^{1/3}  )$} & communicating & \multirow{-2}{*}{$sp(p)$, $T$} & \multirow{-2}{*}{inefficient} \\

     \cline{2-8}
    \multirow{-8}{*}{\rotatebox{90}{\parbox[c]{4cm}{\centering frequentist}}} & Alg. 4 &&& weakly &   & \\

    & \cite{chen_2022_optimisQlearn} & \multirow{-2}{*}{$\tilde{O}(sp(p)S \sqrt{AT} )$} & \multirow{-2}{*}{$\tilde{O}( sp(p)S \sqrt{AT} )$} & communicating & \multirow{-2}{*}{$sp(p)$, $T$} & \multirow{-2}{*}{intractable} \\ 

     \cline{0-7}

     &\textsc{CMDP-PSRL}&&&&&&  \\
    & \cite{Agarwal_2021_PSRL} &  \multirow{-2}{*}{$\tilde{O}( T_M S\sqrt{AT})$} & \multirow{-2}{*}{$\tilde{O}( T_M S \sqrt{AT})$} & \multirow{-2}{*}{ergodic} & \multirow{-2}{*}{-} & \multirow{-2}{*}{efficient} \\ 
   
   \cline{2-8}
    \multirow{-4}{*}{\rotatebox{90}{\parbox[c]{2.55cm}{\centering Bayesian}}} & \textsc{PSConRL}  &&&&&& \\
    & (this paper) &  \multirow{-2}{*}{$\tilde{O}( DS \sqrt{AT})$} & \multirow{-2}{*}{$\tilde{O}( DS \sqrt{AT})$} & \multirow{-2}{*}{communicating} & \multirow{-2}{*}{-} & \multirow{-2}{*}{efficient} \\
   \hline
   \hline
    &lower bound&&&&&&  \\
    & \cite{Singh_CMDP_2020} &  \multirow{-2}{*}{$\Omega( \sqrt{DSAT})$} & \multirow{-2}{*}{$\Omega( \sqrt{DSAT})$} & \multirow{-2}{*}{-} & \multirow{-2}{*}{-} & \multirow{-2}{*}{-} \\
    \hline
  \end{tabular}
   \end{threeparttable}}
\label{tbl:extended_comparison}
\vspace*{-1.2\baselineskip}
\end{table*}

Learning in CMDP in the \textit{infinite-horizon average reward} setting appears to be more challenging because it depends on the limiting behavior of the underlying stochastic process, and approaches for analyzing this setting vary with the class structure of CMDPs. 
For instance,  \citet{Singh_CMDP_2020} and \citet{pmlr-v120-zheng20a} consider a restricted class of \textit{ergodic} CMDPs. In ergodic CMDPs, any policy will reach every state after a sufficient number of steps, making them self-explorative and easier to learn than general cases. Nevertheless, achieving near-optimal regret bounds is still non-trivial under constraints, and the proposed algorithms only achieve suboptimal regret bounds: with \textsc{UCRL-CMDP} \cite{Singh_CMDP_2020} achieving $\Tilde{O}(T^{2/3})$ regret and cost violation bound and \textsc{C-UCRL} \citep{pmlr-v120-zheng20a} achieving $\Tilde{O}(T^{3/4})$ regret bound with no cost violations. 
In contrast, \citet{chen_2022_optimisQlearn} consider a broad class of \textit{weakly communicating} CMDPs, which allows more interesting practical scenarios. They propose two algorithms in this more general setting, albeit imposing impractical assumptions about knowledge of some problem-specific parameters. The first algorithm is computationally tractable but theoretically suboptimal, only achieving $\Tilde{O}(T^{2/3})$ regret and cost violation bounds; the second is an intractable algorithm with near-optimal regret and cost violation bounds of $\tilde{O} ( \sqrt{T} )$. The main theoretical results for this setting are summarized in Table \ref{tbl:extended_comparison}.

In this paper, we propose a practical \textit{and} efficient algorithm based on the posterior sampling principle \cite{THOMPSON_1933}. This principle involves maintaining a posterior distribution for the unknown parameters and guides the exploration by the variance of the distribution. The posterior sampling principle underpins many algorithms in reinforcement learning \cite{Osband_PSRL2013, AY_bayesian_control_2015, NIPS2017_Shipra_OPSRL, TS_MDP_Ouyang_2017}.

\textbf{Our main contribution} is a posterior sampling-based algorithm (\textsc{PSConRL}), which achieves near-optimal Bayesian regret bounds while being computationally efficient. Drawing inspiration from the algorithmic design structure of \citet{TS_MDP_Ouyang_2017}, the algorithm proceeds in episodes with two stopping criteria. At the beginning of every episode, it samples transition probability vectors from a posterior distribution for every state-action pair. The key idea of the algorithm is to switch to efficient exploration whenever the sampled transitions are infeasible, which we show to be necessary for communicating CMDPs. When sampled transitions are feasible, the algorithm solves for the optimal policy by utilizing a linear program (LP) in the space of occupancy measures that incorporates constraints directly \cite{Altman99constrainedmarkov}. The optimal policy computed for the sampled CMDP is used throughout the episode. Under a Bayesian framework, we show that the expected regret and cost violation of our algorithm accumulated up to time $T$ is bounded by $\tilde{O} ( DS\sqrt{AT} )$ for any communicating CMDP with $S$ states, $A$ actions, and diameter $D$. 

A closely related study by \citet{Agarwal_2021_PSRL} analyzes the long-term average Bellman error in constrained optimization to address potential infeasibility issues of posterior sampling. They achieve the Bayesian regret and cost violation bounds of $\tilde{O} (T_MS\sqrt{AT})$, where $T_M$ is the mixing time.\footnote{Mixing time can be arbitrarily loose compared to the diameter, e.g., $T_M \sim O(D^S)$ for some problem instances \cite{regal_2009}.} However, they focus on the ergodic CMDP structure, and, as detailed in Sections \ref{subsec:counterexample}, their method cannot be applied to communicating CMDPs. 

Thus, the main result of the paper shows that near-optimal Bayesian regret bounds are achievable in constrained reinforcement learning. To the best of our knowledge, this is the first work to obtain a computationally tractable algorithm with near-optimal regret bounds for the infinite-horizon average reward setting when underlying CMDP is communicating. Additionally, simulation results demonstrate that our algorithm significantly outperforms existing approaches for three CMDP benchmarks.

The rest of the paper is organized as follows.  Section \ref{sec:problem_formulation} is devoted to the methodological setup and contains the problem formulation.  The \textsc{PSConRL} algorithm is introduced in Section \ref{sec:algorithm}. Analysis of the algorithm is presented in Section \ref{sec:regret_bound}, which is followed by numerical experiments in Section \ref{sec:experiments}. Section \ref{sec:lit_review} briefly reviews the previous related work. Finally, we conclude with Section \ref{sec:conclusion}.

\section{Problem formulation}
\label{sec:problem_formulation}
\subsection{Constrained Markov Decision Processes}
A constrained MDP model is defined as a tuple $M = (\mathcal{S}, \mathcal{A}, p, \textit{\textbf{c}}, \tau)$ where $\mathcal{S}$ is the state space, $\mathcal{A}$ is the action space, $p : \mathcal{S} \times \mathcal{A} \xrightarrow{} \Delta^{\mathcal{S}}$ is the transition function, with $\Delta^{\mathcal{S}}$ indicating simplex over $\mathcal{S}$, $\textit{\textbf{c}} : \mathcal{S} \times \mathcal{A} \xrightarrow{} [0, 1]^{m+1}$ is the cost vector function, and $\tau \in [0,1]^m$ is a cost threshold. In general, CMDP is an MDP with multiple cost functions ($c_0, c_1,\dots,c_m$), one of which, $c_0$, is used to set the optimization objective, while the others, ($c_1,\dots,c_m$), are used to restrict what policies can do. A stationary policy $\pi$ is a mapping from state space $\mathcal{S}$ to a probability distribution on the action space $\mathcal{A}$, $\pi : \mathcal{S} \xrightarrow{} \Delta^{\mathcal{A}}$, which does not change over time.  Let $S = |\mathcal{S}|$ and $A = |\mathcal{A}|$, where $|\cdot|$ denotes the cardinality.

For transitions $p$ and a scalar cost function $c$, a stationary policy $\pi$ induces a Markov chain, and the expected infinite-horizon average cost (\textit{loss}) for state $s \in \mathcal{S}$ is defined as
\begin{equation}
    J^{\pi}(s; c, p) = \overline{\lim}_{T \to \infty} \frac{1}{T} \sum_{t=1}^T \mathbb{E}_{p}^{\pi} \left [ c(s_t,a_t) \lvert s_0 = s \right ],
\end{equation}
where $\mathbb{E}_{p}^{\pi}$ is the expectation under the probability measure $\mathbb{P}_{p}^{\pi}$ over the set of infinitely long state-action trajectories. $\mathbb{P}_{p}^{\pi}$ is induced by policy $\pi$, transition function $p$, and the initial state $s$. Given some fixed initial state $s$  and  $\tau_1, \dots , \tau_m \in \mathbb{R}$ , the CMDP optimization problem is to find a policy  $\pi$  that minimizes $J^\pi(s; c_0,p)$ subject to the constraints  $J^\pi(s; c_i,p) \leq \tau_i, i = 1, \dots, m$:
\begin{equation}
    \min_\pi J^\pi(s; c_0, p)  \text{ s.t. }  J^\pi(s; c_i, p) \leq \tau_i,\, i=1,\dots,m \,. 
    \label{eq:objective_cost}
\end{equation}

\textbf{Communicating CMDPs.} To control the regret vector (defined below), we consider the subclass of communicating CMDPs. Formally, define the diameter of CMDP with transitions $p$ as the minimum time required to go from one state to another in the CMDP using some stationary policy:
\begin{equation*}
    D(p) = \max_{s \neq s^{\prime}} \min_{\pi:\mathcal{S} \to \Delta^{\mathcal{A}}} T^{\pi}_{s \to s^{\prime}},
\end{equation*}
where $T^{\pi}_{s \to s^{\prime}}$ is the expected number of steps to reach state $s^{\prime}$ when starting from state $s$ and using policy $\pi$. CMDP is communicating if and only if it has a finite
diameter, that is to say, for every pair of states $s$ and $s^{\prime}$ there exists a stationary policy under which $s^{\prime}$ is accessible from $s$ in at most $D(p)$ steps, for some finite $D(p) \geq 0$.

We define $\Omega_*$ to be the set of all transitions $p$ such that the CMDP with transition probabilities $p$ is communicating, and there exists a number $D$ such that $D(p) \leq D$. We will focus on CMDPs
with transition probabilities in set $\Omega_*$.

Next, by \cite{Puterman_mdp}[Theorem 8.2.6], for scalar cost function $c$, transitions $p$ that corresponds to communicating CMDP, and stationary policy $\pi$, there exists a bias function $v(s;c,p)$ satisfying the \textit{Bellman equation} for all $s \in \mathcal{S}$:
\begin{align}
\label{eq:bellman}
    & J^{\pi}(s; c, p) + v^{\pi}(s; c, p) \\
    & = \sum_{a \in \mathcal{A}} \pi(a|s) \left [c(s,a) + \sum_{s' \in \mathcal{S}} p(s'|s,a) v^{\pi}(s'; c, p) \right ]. \notag
\end{align}
If $v$ satisfies the Bellman equation, $v$ plus any constant also satisfies the Bellman equation. Furthermore, the loss of the optimal stationary policy $\pi_{\ast}$ does not depend on the initial state, i.e., $J^{\pi_{\ast}}(s; c, p) = J^{\pi_{\ast}}(c, p)$, as presented in \cite{Puterman_mdp}[Theorem 8.3.2]. Without loss of generality, let $\min_{s \in \mathcal{S}} v^{\pi_*}(s; c_i, p) = 0$, for $i=1, \dots m$, and define the span of the MDP as $sp(p) = \max_{1 \leq i \leq m}\max_{s \in \mathcal{S}}v^{\pi_*}(s; c_i, p)$. Note, if $D(p) \leq D$, then $sp(p) \leq D$ as well \citep{regal_2009}.

\textbf{Linear programming for solving CMDPs.} When CMDP is known, an optimal policy for \eqref{eq:objective_cost} can be obtained by solving the following linear program (LP)\citep{Altman99constrainedmarkov}:
\begin{align}
    \min_{\mu} \sum_{s,a} \mu(s,a) c_0(s,a), \label{eq1}\\
    \mathrm{s.t.}\quad \sum_{s,a} \mu(s,a) c_i(s,a) \leq \tau_i,\, \quad i=1,\dots,m, \\
    \sum_a \mu(s,a) = \sum_{s', a} \mu(s', a) p(s',a,s), \quad \forall s \in \mathcal{S}, \\
    \mu(s,a) \geq 0, \quad \forall (s,a) \in \mathcal{S} \times \mathcal{A}, \quad \sum_{s,a} \mu(s,a) = 1, \label{eq4}
\end{align}
where the decision variable $\mu(s,a)$ is occupancy measure (fraction of visits to $(s,a)$). Given the optimal solution for LP \eqref{eq1}-\eqref{eq4}, $\mu_*(s,a)$, one can construct the optimal stationary policy $\pi_{\ast}(a|s)$ for \eqref{eq:objective_cost} by choosing action $a$ in state $s$ with probability $\frac{\mu_*(s,a)}{\sum_{a'} \mu_*(s,a')}$.

Given the above definitions and results, we can now define the reinforcement learning problem studied in this paper.

\subsection{The reinforcement learning problem}

We study the reinforcement learning problem where an agent interacts with a communicating CMDP $M = (\mathcal{S}, \mathcal{A}, p_{\ast}, \textit{\textbf{c}}, \tau)$. We assume that the agent has complete knowledge of $\mathcal{S}, \mathcal{A}$, and the cost function $\textit{\textbf{c}}$, but not the transitions $p_{\ast}$ or the diameter $D$. This assumption is common for RL literature \citep{regal_2009, NIPS2017_Shipra_OPSRL, whyPSbetter, kalagarla2023safe} and is without loss of generality because the complexity of learning the cost and reward functions is dominated by the complexity of learning the transition probability.

We focus on a Bayesian framework for the unknown parameter $p_{\ast}$. That is, at the beginning of the interaction, the actual transition probabilities $p_{\ast}$ are randomly generated from the prior distribution $f_1$. The agent can use past observations to learn the underlying CMDP model and decide future actions. The goal is to minimize the total cost $\sum_{t=1}^T c_0(s_t, a_t)$ while violating constraints as little as possible, or equivalently, minimize the total regret for the main cost component and auxiliary cost components over a time horizon $T$, defined as
\begin{align*}
    BR_+ (T;c_0) & = \mathbb{E} \left [  \sum_{t=1}^T \big (c_0(s_t, a_t) - J^{\pi_{\ast}}(c_0; p_{\ast}) \big )_+\right ], \\
    BR_+ (T;c_i) & = \mathbb{E} \left [  \sum_{t=1}^T \big (  c_i(s_t, a_t) - \tau_i \big )_+ \right ], i=1,\dots,m,
\end{align*}
where $s_t, a_t, t = 1, \dots , T$, are generated by the agent, $J^{\pi_{\ast}}(c_0; p_{\ast})$ is the optimal loss of the CMDP $M$, and $[x]_+ := \max \{0,x\}$. The above expectation is with respect to the prior distribution $f_1$, the randomness in the state transitions, and the randomized policy.

\subsection{Assumptions}
We introduce two mild assumptions that are common in reinforcement learning literature.

\begin{assumption}
The support of the prior distribution $f_1$ is a subset of $\Omega_*$. That is, the CMDP $M$ is communicating and $D(p_{\ast}) \leq D$.
\label{assum:WASP}
\end{assumption}

This type of assumption is common for the Bayesian framework (see, e.g., \citep{TS_MDP_Ouyang_2017, Agarwal_2021_PSRL}) and is not overly restrictive \citep{regal_2009, chen_2022_optimisQlearn}. In the experiments section, we provide a practical justification for this assumption and show that it can be supported by choosing Dirichlet distribution as a prior.

\begin{assumption}
\label{assum:slater}
    There exists $\gamma > 0$ and unknown policy $\bar{\pi}(\cdot | s) \in \Delta^{\mathcal{A}}$ such that $J^{\bar{\pi}}(c_i, p_{\ast}) \leq \tau_i - \gamma$ for all $i \in \{1, \dots, m \}$, and without loss of generality, we assume under such policy $\bar{\pi}$, the Markov chain resulting from the CMDP is irreducible and aperiodic.
\end{assumption}

The first part of the assumption is standard in constrained reinforcement learning (see, e.g., \citep{Efroni_2020_CMDP, DingWYWJ21}) and is mild as we do not require the knowledge of such policy. The second part is without loss of generality due to \citet{Puterman_mdp}[Proposition 8.3.1] and \citet{Puterman_mdp}[Proposition 8.5.8]. By imposing this assumption, we can control the sensitivity of problem \eqref{eq:objective_cost} to the deviation between the true and sampled transitions. Later, we will use this assumption to guarantee that the minimization problem in Eq. \eqref{eq:objective_cost} becomes feasible under the sampled transitions.

\section{\textsc{PSConRL}: Learning algorithm for constrained reinforcement learning}
\label{sec:algorithm}
In this section, we propose the Posterior Sampling for Constrained Reinforcement Learning (\textsc{PSConRL}) algorithm. Our algorithm is based on an intuitive idea of constructing an adaptive exploration mechanism to address the feasibility issues. It maintains posteriors for the transition function and combines the steps of solving LP through the lens of occupancy measure with the construction of exploration MDPs (whenever LP is infeasible). Below, we describe the main components of our algorithm, which is summarized in Algorithm \ref{alg1:psrl_transitions}.

\textbf{Bayes rule.}
At each timestep $t$, given history $h_t$, the agent can compute posterior distribution $f_t$ given by $f_t(\mathcal{P}) = \prob(p_{\ast} \in \mathcal{P} \lvert h_t)$ for any set $\mathcal{P}$. Upon applying action $a_t$ and observing a new state $s_{t+1}$, the posterior distribution at $t + 1$ can be updated according to Bayes’ rule as
\begin{equation}
\label{eq:bayes_rule}
    f_{t+1}(dp) = \frac{p(s_{t+1} | s_t, a_t) f_t(dp)}{\int p^{\prime}(s_{t+1}|s_t, a_t) f_t(dp^{\prime})}.
\end{equation}
The key challenge of posterior sampling is that neither problem in Eq. \eqref{eq:objective_cost} nor LP \eqref{eq1}-\eqref{eq4} are guaranteed to be feasible under the sampled transitions $p_t \sim f_t$, and it is unclear how the agent should proceed if LP \eqref{eq1}-\eqref{eq4} is infeasible. As we show in Lemma \ref{lm:feasibility}, after sufficient exploration, LP \eqref{eq1}-\eqref{eq4} becomes feasible with high probability (when each state-action pair is visited $\sqrt{T/A}$ times). Therefore, whenever LP \eqref{eq1}-\eqref{eq4} is infeasible, the agent switches to efficient exploration by constructing shortest path policies for a set of MDPs described below.

\textbf{Reduction to a set of exploration MDPs.} To facilitate efficient exploration, we introduce a set of MDPs, denoted as $\{(\mathcal{S}, \mathcal{A}, p_t, c_s)\}_{s \in \mathcal{S}}$. Each MDP in this set retains the original state and action spaces, with the transition function \(p_t \sim f_t\) and a state-dependent cost function \(c_s\), defined as
\begin{equation*}
    c_s(s^{\prime}, a) = \begin{cases}
    1, &\text{if $s^{\prime} \neq s$;}\\
    0, &\text{otherwise.}
    \end{cases}
\end{equation*}

Consider a specific target state $\Bar{s}$ and its corresponding MDP $M^t_{\Bar{s}} = (\mathcal{S}, \mathcal{A}, p_t, c_{\Bar{s}})$. Note, MDP $M^t_{\Bar{s}}$ is communicating with a scalar cost function, and, from MDP theory, we know that there exists an optimal policy $\pi^t_{\Bar{s}}$ that satisfies the Bellman optimality equation:
\begin{align}
\label{eq:opt_bellman}
    & J^*(c_{\Bar{s}}, p_t) + v^*(s; c_{\Bar{s}}, p_t) \\
    & = \min_{a \in \mathcal{A}} \left \{ c_{\Bar{s}}(s,a) + \sum_{s' \in \mathcal{S}} p_t(s'|s,a) v^*(s'; c_{\Bar{s}}, p_t) \right \}, \forall s \in \mathcal{S}. \notag
\end{align}
In essence, the optimal policy $\pi^t_{\Bar{s}}$ corresponds to a policy that efficiently guides the agent through the MDP toward the target state $\Bar{s}$, thereby enabling efficient exploration. The formalization of this intuitive concept will be presented in Section \ref{sec:regret_bound}.

\subsection{Algorithm description} 
\textsc{PSConRL} begins with a prior distribution over transitions $f_1$ and proceeds in episodes. Let $N_t(s,a)$ denote the number of visits to $(s,a)$ before time $t$ and $N_t(s)$ denote the number of visits to $s$. We use two stopping criteria of \citet{TS_MDP_Ouyang_2017} for episode construction. The rounds $t= 1,...,T$ are broken into consecutive episodes as follows: the $k$-th episode begins at the round $t_k$ immediately after the end of $(k - 1)$-th episode and ends at the first round $t$ such that (i) $N_{t}(s,a) \geq 2 N_{t_k}(s,a)$ or (ii) $t \leq t_k + T_{k-1}$ for some state-action pair $(s,a)$, where $T_k = t_{k+1} - t_k$ is the length of episode $k$. The first criterion is the doubling trick of \citet{JMLR:v11:jaksch10a} and ensures the algorithm has visited some state-action pair $(s,a)$ at least the same number of times it had visited this pair $(s,a)$ before episodes $k$ started. The second criterion controls the growth rate of episode length and is believed to be necessary under the Bayesian setting \citep{TS_MDP_Ouyang_2017}.

At the beginning of episode $k$, a parameter $p_k$ is sampled from the posterior distribution $f_{t_k}$, where $t_k$ is the start of the $k$-th episode. During each episode $k$, actions are generated from the optimal stationary policy $\pi_k$ for the sampled parameter $p_k$, which is observed either by solving LP \eqref{eq1}-\eqref{eq4} (if it is feasible) or by recovering the shortest path policy for a state with minimum visitations to it. 
Using Assumption \ref{assum:slater}, we will show that eventually, after $O ( \sqrt{T} )$ steps, the sampled CMDP becomes feasible, and the algorithm will effectively compute $\pi_k$ by solving LP \eqref{eq1}-\eqref{eq4}.

\begin{remark}
    Note that \textsc{PSConRL} only requires the knowledge of $\mathcal{S}$, $\mathcal{A}$, $\textbf{\textit{c}}$, and the prior distribution $f_1$. It does not require the knowledge of the horizon $T$, or the bias span $sp(p)$ as in \citet{Singh_CMDP_2020} and \citet{chen_2022_optimisQlearn}.
\end{remark}

\begin{algorithm}[tb]
   \caption{Posterior Sampling for Constrained Reinforcement Learning (\textsc{PSConRL})}
   \label{alg1:psrl_transitions}
\begin{algorithmic}[1]
    \STATE {\bfseries Input:} $f_1$

    \STATE \textbf{Initialization:} $t \gets 1$, $t_k \gets 0$, $\pi_0(\cdot) \gets \frac{1}{\lvert \mathcal{A} \rvert}$
    \FOR{episodes $k = 1, 2, \dots$}
        \STATE $T_{k-1} \gets t - t_k$ 
        \STATE $t_k \gets t$
        \STATE Generate $p_k(\cdot|s,a) \sim f_{t_k}$
        \IF{ LP (\ref{eq1})-(\ref{eq4}) is feasible under $p_k(\cdot|s,a)$}
            \STATE Compute $\pi_k(\cdot)$ by solving LP (\ref{eq1})-(\ref{eq4})
        \ELSE
            \STATE Select $s$ with minimum number of visits $N_{t_k}(s)$ 
            \STATE Compute $\pi_k(\cdot)$ by solving Eq. \eqref{eq:opt_bellman} for MDP $M_s$
        \ENDIF
        \REPEAT
            \STATE Apply action $a_t = \pi_k(s_t)$
            \STATE Observe new state $s_{t+1}$
            \STATE Update counter $N_t(s_t,a_t)$
            \STATE Update $f_{t+1}$ according to Eq. \eqref{eq:bayes_rule}
            \STATE $t \gets t + 1$
        \UNTIL{ $t \leq t_k + T_{k-1}$ and $N_t(s,a) \leq 2 N_{t_k}(s,a)$} for some $(s,a) \in \mathcal{S} \times \mathcal{A}$
    \ENDFOR
\end{algorithmic}
\end{algorithm}

\subsection{Importance of additional exploration for communicating CMDPs}
\label{subsec:counterexample}
In this subsection, we highlight the importance of reducing the problem to the exploration MDPs within our algorithm. In contrast to our approach, \textsc{PSRL-CMDP} \citep{Agarwal_2021_PSRL} exclusively solves LP \eqref{eq1}-\eqref{eq4} for the optimal solution, and in cases when the optimal solution is infeasible, they opt to disregard constraints and proceed with the unconstrained problem. They argue that, eventually, the LP becomes feasible due to the self-exploratory properties of ergodic CMDPs. Unfortunately, this argument does not hold for communicating CMDPs, as demonstrated by the following example.

\begin{figure*}[htpb]
\hfill
\subfigure[Symbolic illustration of Example \ref{counterexample}.]{
    \includegraphics[width=.29\textwidth]{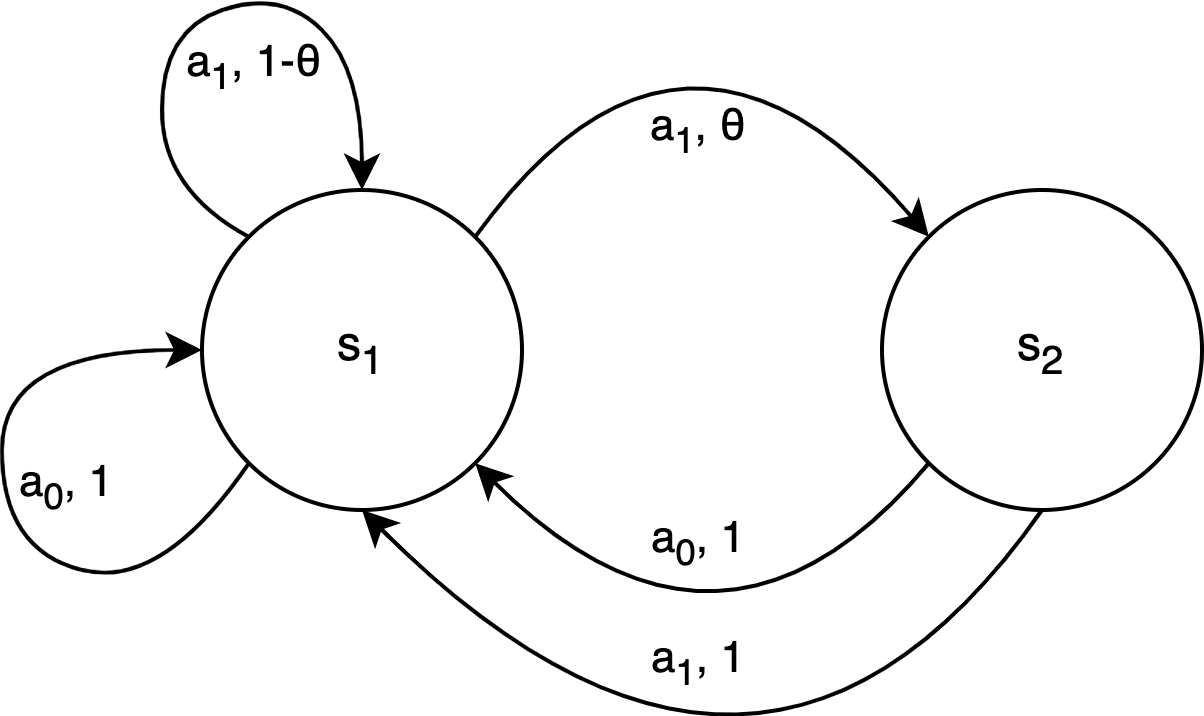}
    \label{fig:counterexample_cmdp}}
\hfill
\subfigure[Simulation results for Example \ref{counterexample}.]{
    \includegraphics[width=.6\textwidth]{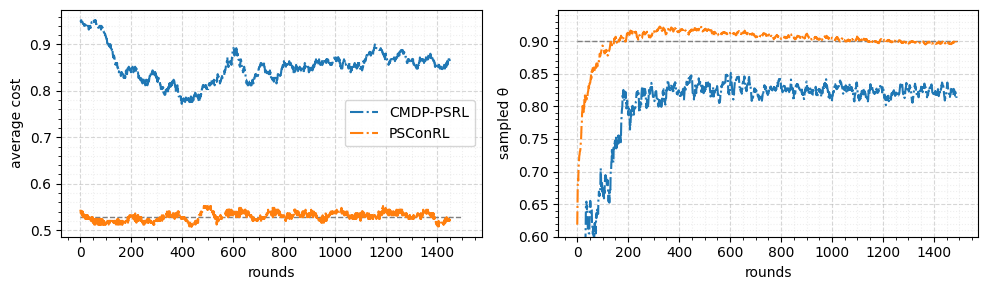}
    \label{fig:counterexample_simulations}}
\hfill
\caption{CMDP illustration and results of the experiments for Example \ref{counterexample}, with $\theta=0.9$ and the average cost threshold $\tau=0.5275$. Figure \ref{fig:counterexample_cmdp} represents the CMDP in symbolic form. Figure \ref{fig:counterexample_simulations} presents average cost (left), and realizations of  $\tilde{\theta}$ (right). Results are averaged over 5 runs.}
\vspace*{-1.\baselineskip}
\end{figure*}

\begin{example}
    \label{counterexample}
    Consider a two-state $\mathcal{S} = \{s_0, s_1\}$, two-action $\mathcal{A} =\{a_0, a_1\}$ CMDP in which the controlled transition probabilities $p_{\ast}(s_0, a_1, s_1) = \theta$ and $p_{\ast}(s_0, a_1, s_0) = 1 - \theta$ are unknown, while remaining probabilities are $p_{\ast}(s_0, a_0, s_1) = 1$, $p_{\ast}(s_1, \cdot, s_1) = 1$ and known. See Figure \ref{fig:counterexample_cmdp} for illustration. Assume that $r(s_0, \cdot)=1$, $c(s_0, \cdot) = 1$ and $r(s_1, \cdot)=0$, $c(s_1, \cdot) = 0$, i.e., reward and cost depend only upon the current state. Further, let $\theta=0.9$ and the average cost threshold $\tau=0.5275$ (the lowest possible budget that corresponds to a feasible problem). Note that this CMDP is not ergodic because starting from $s_0$ and utilizing a policy that chooses action $a_0$ will never visit state $s_1$. Also, such a policy would clearly correspond to an optimal solution in case there is no budget constraint.
\end{example} 

For two algorithms, \textsc{PSConRL} (ours) and \textsc{PSRL-CMDP} \citep{Agarwal_2021_PSRL}, we demonstrate their performance through simulations on this toy CMDP. For both algorithms, we set the prior distribution to $Beta(1, 1)$ with the parameters of the distribution being the number of visitations to $(s_0, a_0)$ and $(s_0, a_1)$ state-action pairs. At each timestep, we sample the plausible parameter $\tilde{\theta}$. Whenever the sampled CMDP is infeasible, we utilize the optimal policy for the unconstrained problem for \textsc{PSRL-CMDP} (policy that chooses action $a_0$ all the time) and the shortest path policy according to Eq. \eqref{eq:opt_bellman} for \textsc{PSConRL} (policy that chooses action $a_1$ all the time). Note that the sampled CMDP is infeasible every time $\tilde{\theta} < \theta$, due to the choice of cost threshold. 

Figure \ref{fig:counterexample_simulations} demonstrates the results of the experiment. Specifically, we present the average cost (left) and realizations of $\tilde{\theta}$ (right). Taking a closer look at the average cost subplot (left), we can see that \textsc{PSConRL} consistently fluctuates around the cost threshold and, overall, satisfies the constraint of the problem, whereas \textsc{PSRL-CMDP} severely violates the constraint. Moving to the right subplot, it is evident that \textsc{PSConRL} successfully learns the true value of parameter $\theta$, while \textsc{PSRL-CMDP} fails to do so. 

A series of assumptions in \cite{Agarwal_2021_PSRL} makes a Markov chain induced by any policy aperiodic, recurrent, and irreducible. Such favorable properties make any CMDP self-exploratory, meaning that for a sampled CMDP, a policy that solely maximizes the main reward (regardless of constraints) will sufficiently explore the environment and, eventually, collect enough information to find the true optimal solution. However, this does not hold for communicating CMDPs and necessitates additional exploration to ensure feasibility. As such, a more involved theoretical analysis is required to address this issue for communicating CMDPs.

\section{Regret bound}
We now provide our main result for the \textsc{PSConRL} algorithm for learning in CMDPs.

\begin{theorem}
\label{thm:regret_bound}
     For any communicating CMDP $M$ with $S$ states, $A$ actions, under Assumptions \ref{assum:WASP} and \ref{assum:slater}, for $T \geq \Omega((D/\gamma)^4 S^2 A \log^2(2AT))$, the Bayesian regret for main and auxiliary cost components of Algorithm \ref{alg1:psrl_transitions} are bounded:
    \begin{align*}
        & BR_+ (T;c_i) \leq O \left ( DS \sqrt{AT \log (AT)} \right ), i = 0, \dots, m. 
    \end{align*}
    Here $O(\cdot)$ notation hides only the absolute constant.E
\end{theorem}

\begin{remark}
    The regret bound closely matches the theoretical lower bound of $\Omega(\sqrt{DSAT})$. Also, the provided bound matches the best bound for the undiscounted setting without constraints. We emphasize that the $O(\sqrt{DS})$ gap between lower and upper bounds remains an open question for the undiscounted setting with and without constraints. 
\end{remark}

\label{sec:regret_bound}


The full proof of Theorem \ref{thm:regret_bound} is presented in the Appendix \ref{app:main_thm_proof}. In the remainder of this section, we introduce three lemmas that are pivotal to our analysis and present a proof sketch for Theorem \ref{thm:regret_bound}.

\subsection{Key lemmas}

A key property of posterior sampling is that conditioned on the information at time $t$, the transition functions $p_{\ast}$ and $p_t$ have the same distribution if $p_t$ is sampled from the posterior distribution at time $t$ \cite{Osband_PSRL2013}. Since the \textsc{PSConRL} algorithm samples $p_k$ at the stopping time $t_k$, we use the stopping time version of the posterior sampling property stated as follows.

\begin{lemma}[Posterior sampling lemma; adapted from Lemma 1 of \cite{jafarniajahromi2021online}]
Let $t_k$ be a stopping time with respect to the filtration $\left ( \mathcal{F}_t \right )_{t=1}^{\infty}$, and $p_k$ be the sample drawn from the posterior distribution at time $t_k$. Then, for any measurable function $g$ and any $\mathcal{F}_{t_k}$-measurable random variable $X$, we have 
\label{lm:posterior_lemma}
\begin{align*}
    \ee \left [ g(p_k, X) \right ] = \ee \left [ g(p_{\ast}, X) \right ].
\end{align*}
\end{lemma}
Recall that in every episode $k$, \textsc{PSConRL} runs either an optimal loss policy by solving LP \eqref{eq1}-\eqref{eq4} for the sampled transitions or computes the optimal stationary policy for a fixed finite MDP. In Lemma \ref{lm:feasibility}, we show that problem \eqref{eq:objective_cost} becomes feasible under sampled transitions after sufficient exploration of every state-action pair, i.e., there exists a policy that satisfies constraints in problem \eqref{eq:objective_cost} and Algorithm~\ref{alg1:psrl_transitions} will effectively find an optimal solution for LP \eqref{eq1}-\eqref{eq4}. 

We address the feasibility issues by using the deviation bound between sampled and true transitions and the limiting matrix properties of the resulting Markov chains. Unlike optimistic algorithms \cite{Singh_CMDP_2020, chen_2022_optimisQlearn}, which optimize over a confidence set of plausible transitions, Lemma \ref{lm:feasibility} introduces a computationally efficient approach to deal with feasibility issues.

\begin{lemma}[Feasibility lemma]
\label{lm:feasibility}
    If $N_{t_k}(s,a) \geq \sqrt{T/A}$, $\norm{p_k(\cdot|s,a) - p_{\ast}(\cdot|s,a)}_1 \leq \sqrt{\frac{14S \log (2A T t_k)}{\max \{1, N_{t_k}(s,a)\}}}$ for all $(s,a)$, and $\gamma \geq D \sqrt{\frac{14S A^{1/2} \log (2A T^2)}{\sqrt{T}}}$ there exists policy $\pi$, which satisfies $J^\pi(c_i, p_k) \leq \tau_i$ for all $i \in \{1, \dots, m\}$. 
\end{lemma}

Next, in Lemma \ref{lm:exploration_lemma}, we prove that \textsc{PSConRL} explores the environment efficiently, whenever LP \eqref{eq1}-\eqref{eq4} is infeasible, and requires $O(DS\sqrt{AT})$ to visit each state-action pair $\sqrt{T/A}$ times.

\begin{lemma}[Exploration lemma]
\label{lm:exploration_lemma}
    Define set $\mathcal{G} = \left \{ p \in \Omega_*: \exists \pi \text{ s.t. } J^\pi(c_i; p) \leq \tau_i, \forall i \in \{1, \dots, m\} \right \}$. Whenever $\pi_k$ is computed as an optimal policy for Eq.~\eqref{eq:opt_bellman}, i.e., $ p_k \notin \mathcal{G} $, the average number of timesteps to visit each state-action pair $\sqrt{T/A}$ times is bounded by $2DS\sqrt{AT}+1$. Formally, 
    \begin{align*}
        \sum_{k:t_k \leq T} \ee &  \left [ \sum_{t=t_k}^{t_{k+1}-1} \mathbb{I} \left \{ \exists (s,a):  N_{t_k}(s,a) < \sqrt{T} \right \} \left.\right| p_k \notin \mathcal{G} \right ] \\
        & \leq 2DS\sqrt{AT} + 1.
    \end{align*}
\end{lemma}

Lemma \ref{lm:exploration_lemma} plays a crucial role in facilitating efficient exploration. It ensures that the deviation between sampled and true transitions becomes sufficiently small, thereby satisfying the conditions outlined in Lemma \ref{lm:feasibility}. Importantly, our exploration mechanism requires overall $O(\sqrt{T})$ steps, whereas the existing approaches designate $O(T^{2/3})$ steps for exploration in constrained problems and only achieve suboptimal regret of $\tilde{O}(T^{2/3})$, e.g., UCRL-CMDP \cite{Singh_CMDP_2020} and Alg. 3 \cite{chen_2022_optimisQlearn}. Only Alg. 4 \cite{chen_2022_optimisQlearn} allocates $O(\sqrt{T})$ steps for exploration, leading to near-optimal regret bound. However, this exploration scheme renders their algorithm intractable.

The Feasibility and Exploration lemmas form one of the main novel components of the analysis of Theorem \ref{thm:regret_bound}.

\subsection{Proof Sketch of Theorem \ref{thm:regret_bound}}
     Below, we show a proof sketch of the main theorem for the main regret component. The proof for auxiliary cost components is deferred to the Appendix \ref{app:main_thm_proof}.

    We decompose the total regret into the sum of episodic regrets conditioned on the event that the sampled CMDP is feasible:
    \begin{align}
        BR_+ (T;c_0) 
        & =  \ee \left [ \sum_{t=1}^T \left ( c_0(s_t, a_t) - J^{\pi_*} \right )_+ \right ] 
        \notag
        \\ 
        = \sum_{k=1}^{K_T} & \ee \left [ \sum_{t} \left [ c_0(s_t, a_t) - J^{\pi_{\ast}} \right ] \lvert p_k \in \mathcal{G} \right ] \prob \left ( p_k \in \mathcal{G} \right ) 
        \notag
        \\
        + \sum_{k=1}^{K_T} & \ee \left [ \sum_{t} \left [ c_0(s_t, a_t) - J^{\pi_{\ast}} \right \lvert p_k \notin \mathcal{G} \right ] \prob \left ( p_k \notin \mathcal{G} \right ),
        \notag
    \end{align}
    where $J^{\pi_*} = J^{\pi_*}(c_0; p_{\ast})$ is the optimal loss of CMDP $M$, $K_T$ is the number of episodes, and $\mathcal{G}$ is defined in the statement of Lemma \ref{lm:exploration_lemma}.

    For the first term, conditioned on the good event, $\{ p_k \in \mathcal{G} \}$, the sampled CMDP is feasible, and the standard analysis of \citet{TS_MDP_Ouyang_2017} can be applied. Lemma \ref{lm:regret_on_good_event} shows that this term can be bounded by $(D+1) \sqrt{2SAT \log (T)} + 49DS \sqrt{AT \log (AT)}$.

    As for the second term, we further decompose it conditioned on two events: $A_1 = \{p_k \notin \mathcal{G} \wedge  N_{t_k}(s,a) \geq \sqrt{T/A}, \forall s, a \}$ and $A_2 = \{p_k \notin \mathcal{G} \wedge \exists (s,a):  N_{t_k}(s,a) < \sqrt{T/A}\}$. Using the Feasibility lemma, we then show that $\prob(A_1)$ is bounded by $2/15Tt_k$ for each $k$, and the total regret corresponding to event $A_1$ is negligible.

    Next, conditioned on $A_2$, we can utilize the Exploration lemma and show that $\sum_{k} \ee \left [ \sum_{t} \left [ c_0(s_t, a_t) - J^{\pi_{\ast}} \right \lvert A_2 \right ] \prob \left ( A_2 \right ) < 2DS\sqrt{AT} + 1$, due to the efficient exploration property of our algorithm.

    Putting all bounds together, we obtain the resulting regret bound of:
    \begin{align*}
        BR_+ (T;c_0) & 
        \leq O \left ( DS \sqrt{AT \log (AT)} \right ).
    \end{align*}

\section{Simulation results}
\label{sec:experiments}

\begin{figure*}[htb]
    \centering
    \includegraphics[width=0.99\textwidth]{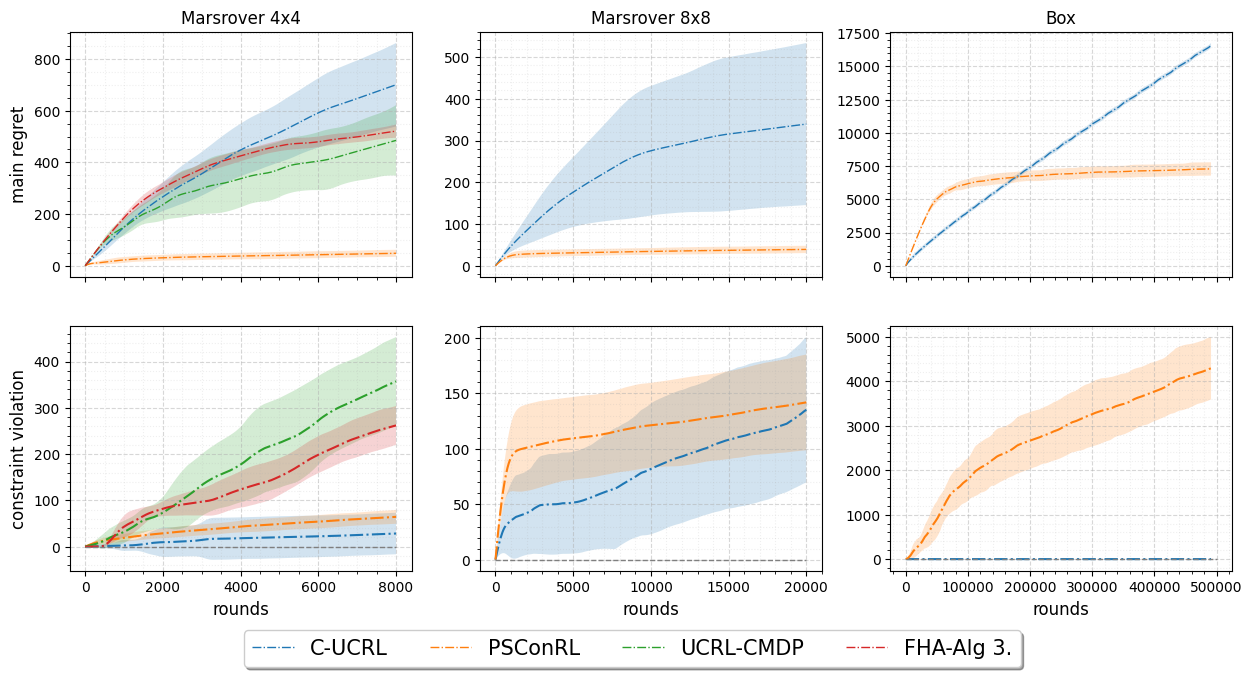}
    \caption{The main regret and constraint violation of the algorithms as a function of the horizon for Marsrover 4x4 (left column), Marsrover 8x8 (middle column), and Box (right column). \textbf{(Top row)} shows the cumulative regret of the main cost component. \textbf{(Bottom row)} shows the cumulative constraint violation. Results are averaged over 50 runs for Marsrover 4x4 and over 30 runs for Marsrover 8x8 and Box. Results for \textsc{UCRL-CMDP} and \textsc{FHA} (Alg. 3) are averaged over 10 runs for Marsrover 4x4.}
    \label{fig:sum_results}
\vspace*{-1.\baselineskip}
\end{figure*}

In this section, we evaluate the performance of \textsc{PSConRL}. The source code of the experiments can be found at \url{https://github.com/danilprov/cmdp}.

We present \textsc{PSConRL} using Dirichlet priors with parameters $[0.1, \dots, 0.1]$. The Dirichlet distribution is a convenient choice for maintaining posteriors for the transition probability vectors $p(s, a)$ since it is a conjugate prior for categorical and multinomial distributions. Moreover, Dirichlet prior is proven to be highly effective for any underlying MDP in unconstrained problems \cite{whyPSbetter}. 

We employ three algorithms as baselines: \textsc{C-UCRL} \citep{pmlr-v120-zheng20a}, \textsc{UCRL-CMDP} \citep{Singh_CMDP_2020}, and \textsc{FHA} (Alg. 3) from \citep{chen_2022_optimisQlearn}. Both Alg. 4 of \citep{chen_2022_optimisQlearn} and \textsc{PSRL-CMDP} of \citep{Agarwal_2021_PSRL} are omitted from the empirical analysis due to their practical inapplicability. For additional information about the baselines, see Appendix \ref{apx:benchmarks}.

We run our experiments on three gridworld environments: Marsrover 4x4, Marsrover 8x8 \cite{pmlr-v120-zheng20a}, and Box \cite{Leike_aisafegrid_2017}. To enable fair comparison, all algorithms were extended to the unknown reward/costs and unknown probability transitions setting (see Appendix \ref{apx:experiments} for more experimental details). Figure \ref{fig:sum_results} illustrates the simulation results of all algorithms across three benchmark environments. The top row shows the cumulative regret of the main cost component. The bottom row presents the cumulative constraint violation. 

We first analyze the behavior of the algorithm on Marsrover environments (left and middle columns). The cumulative regret (top row) shows that \textsc{PSConRL} consistently outperforms all three algorithms. Looking at the cumulative constraint violation (bottom row), we see that \textsc{PSConRL} is comparable with \textsc{C-UCRL}, the only algorithm that addresses safe exploration. In the Box example (right column), \textsc{PSConRL} significantly outperforms \textsc{C-UCRL}, which incurs near-linear regret. We note that exploration is relatively costly in this benchmark compared to Marsrover environments (see the difference on the $x$ and $y$-axes in the top row), which suggests that \textsc{C-UCRL} might be impractical in (at least some) problems where exploration is non-trivial. In Figure \ref{fig:sum_results_app}, we further elaborate on the average performance of the algorithms interpreting regret behavior.

We note the pronounced computational inefficiency of \textsc{UCRL-CMDP} and \textsc{FHA} (Alg. 3) algorithms. 
\textsc{UCRL-CMDP} involves optimization not only across the space of occupancy measures but also across the set of plausible CMDPs, resulting in an exhausting non-linear program. In the most favorable scenario, it requires $O((S^2A)^4)$ operations per episode. On the other hand, \textsc{FHA} (Alg. 3) maintains linearity in its main optimization program but necessitates solving it at each timestep, leading to time complexity of $O((SAT^{1/3})^2)$ per episode. Although both algorithms enjoy polynomial time complexity, the undesirable dependence on problem parameters makes them impractical even for moderate-sized problems. Due to these substantial drawbacks, we have limited their implementation to the Marsrover 4x4 environment.

Additionally, we would like to point out that in these examples, CMDPs are fixed and not generated from the Dirichlet prior. 
Therefore, we conjecture that \textsc{PSConRL} has the same regret bounds under a non-Bayesian setting.

\section{Related work}
\label{sec:lit_review}

Several algorithms based on the \textit{optimism in the face of uncertainty} (OFU) principle have been proposed for constrained RL problems. For the episodic setting, both \citet{Efroni_2020_CMDP} and \citet{NEURIPS2020_Brantley} consider sample efficient exploration utilizing a double optimism principle. As previously mentioned, \citet{Singh_CMDP_2020} and \citet{chen_2022_optimisQlearn} study OFU-based algorithms in the infinite-horizon average reward setting. It is worth mentioning that OFU-based algorithms often involve optimization across a set of plausible models (see, e.g., \cite{Efroni_2020_CMDP, Singh_CMDP_2020}), which makes them computationally less appealing.

Another line of closely related works investigates safe RL, addressing constrained reinforcement learning problems with constant or zero constrain violation guarantees. Several algorithms were proposed in episodic setting \citep{Liu_2021, wei2021provablyefficient, kalagarla2023safe}, with \citet{kalagarla2023safe} focusing on posterior sampling algorithm for safe reinforcement learning. In the infinite-horizon average reward setting, the safe RL problem was previously analyzed in \citep{pmlr-v120-zheng20a, chen_2022_optimisQlearn}. Notably, safe RL algorithms often assume that the transition model and/or safe policy are known.

Among other related work, Lagrangian relaxation is a widely adopted technique for solving CMDPs. The works of \cite{Achiam_CMDP2017, tessler2018reward} present constrained policy optimization approaches that demonstrate prominent successes in artificial environments. However, these approaches are notoriously sample-inefficient and lack theoretical guarantees. More scalable versions of the Lagrangian-based methods were proposed in \cite{Chow_CMDP2018, qiu_2021_cmdp_posterior, chen2021primaldual, provodin2022empirical}. In general, the Lagrangian relaxation method can achieve high performance, but it is sensitive to the initialization of the Lagrangian multipliers and learning rate.

\section{Conclusion}
\label{sec:conclusion}
In this paper, we introduced the \textsc{PSConRL} algorithm for efficient exploration in constrained reinforcement learning under the infinite-horizon average reward criterion. Our algorithm achieves near-optimal Bayesian regret bounds for each cost component while being computationally efficient and easy to implement. By addressing these aspects, \textsc{PSConRL} fills a crucial gap in provably efficient constrained RL. 

\textsc{PSConRL} leverages LP solutions to determine optimal policies and incorporates efficient exploration whenever the sampled CMDP is infeasible. As demonstrated in Section \ref{counterexample}, the empirical comparison between \textsc{PSConRL} and \textsc{CMDP-PSLR} highlights that the exploration step is not merely a technical requirement for proofs but is indeed essential for effective learning in communicating CMDPs.

Finally, we validated our approach using simulations on three gridworld domains and showed that \textsc{PSConRL} quickly converges to the optimal policy even when CMDPs are not sampled from Dirichlet priors, consistently outperforming existing algorithms. Our insights suggest that the use of posterior sampling might be of great value for designing a computationally efficient algorithm with near-optimal frequentist regret bounds. Exploring this direction further is a promising avenue for future work. We also believe that this superior performance extends beyond the scope of gridworld domains to real-life applications.

\section*{Acknowledgements}

This project is partially financed by the Dutch Research Council (NWO) and the ICAI initiative in collaboration with KPN, the Netherlands. The authors thank Pratik Gajane and Thiago D. Simão for discussions on earlier drafts of the paper.

\section*{Impact Statement}

Our work focuses on the theoretical foundations of constrained reinforcement learning, emphasizing practical relevance and applicability. We believe that understanding the theoretical foundations is essential and, when coupled with addressing practically relevant issues, can directly guide the principled and effective application of these methods to real-life problems.

We believe that constraints represent a pivotal limitation in extending RL to real-life problems, such as nuclear fusion, medical treatment, and advertising. Consequently, the potential impact of our work extends to developing the foundations contributing to safe and responsible AI in the context of constrained reinforcement leanring.


\bibliography{main}
\bibliographystyle{icml2024}

\newpage
\appendix
\onecolumn
\section{Omitted details for Section \ref{sec:regret_bound}}
\subsection{Proof of Theorem \ref{thm:regret_bound}}
\label{app:main_thm_proof}
    \paragraph{Bounding regret of the main cost component.}
    To analyze the performance of \textsc{PSConRL} over $T$ time steps, define $K_T = \arg \max \{ k : t_k \leq T \}$, number of episodes of \textsc{PSConRL} until time $T$. By \citet{TS_MDP_Ouyang_2017}[Lemma 1], $K_T$ is upper-bounded by $\sqrt{2SAT \log(T)}$. Using the tower rule, we can decompose the total regret into the sum of episodic regrets conditioned on the good event that the sampled CMDP is feasible:
    \begin{align}
    \label{eq:regret_decompose}
        BR_+ (T;c_0) & 
        = \ee \left [ \sum_{t=1}^T \left ( c_0(s_t, a_t) - J^{\pi_*}(c_0; p_{\ast}) \right )_+ \right ] 
        = \sum_{k=1}^{K_T} \ee \left [ R_{0,k} \right ] 
        \notag
        \\
        & 
        = \sum_{k=1}^{K_T} \ee \left [ R_{0,k} \lvert p_k \notin \mathcal{G} \right ] \prob \left ( p_k \notin \mathcal{G} \right )
        + \sum_{k=1}^{K_T} \ee \left [ R_{0,k} \lvert p_k \in \mathcal{G} \right ] \prob \left ( p_k \in \mathcal{G} \right ),
        \\
        \notag
    \end{align}
    where $R_{0,k} = \sum_{t=t_k}^{t_{k+1}-1} \left [ c_0(s_t, a_t) -  J^{\pi_{\ast}}(c_0; p_{\ast}) \right ]_+$, $J^{\pi_*}(c_0; p_{\ast})$ is the optimal loss of CMDP $M$, and $\mathcal{G}$ is defined in the statement of Lemma \ref{lm:exploration_lemma}.
    
    Define two events $A_1 = \{p_k \notin \mathcal{G} \wedge  N_{t_k}(s,a) \geq \sqrt{T/A}, \forall s, a \}$ and $A_2 = \{p_k \notin \mathcal{G} \wedge \exists (s,a):  N_{t_k}(s,a) < \sqrt{T/A}\}$. Then, the first term of \eqref{eq:regret_decompose} can be further decomposed as
    \begin{align*}
        \sum_{k=1}^{K_T} \ee \left [ R_{0,k} \lvert p_k \notin \mathcal{G} \right ] \prob \left ( p_k \notin \mathcal{G} \right ) 
        & = \sum_{k=1}^{K_T} \ee \left [ R_{0,k} \lvert A_1 \right ] \prob \left ( A_1 \right ) 
        + \sum_{k=1}^{K_T} \ee \left [ R_{0,k} \lvert A_2 \right ] \prob \left ( A_2 \right ). \\
    \end{align*}
    First, we bound $\sum_{k=1}^{K_T} \ee \left [ R_{0,k} \lvert A_1 \right ] \prob \left ( A_1 \right )$. Let $\bar{p}_k(s^{\prime}|s,a) = \frac{N_{t_k}(s,a,s^{\prime})}{N_{t_k}(s,a)}$ be the empirical mean for the transition probability at the beginning of episode $k$, where $N_{t_k}(s,a,s^{\prime})$ is the number of visits to $(s,a,s^{\prime})$. Define the confidence set 
    \begin{align*}
        B_k = \left \{ p : \norm{\bar{p}_k(\cdot|s,a) - p(\cdot|s,a)}_1 \leq \beta_k \right \},
    \end{align*}
    where $\beta_k = \sqrt{\frac{14S \log (2A T t_k)}{\max \{1, N_{t_k}(s,a)\}}}$. 
    
    Now, we observe that $ \left \{ A_1 \right \} \subseteq \left \{ \norm{p_k(\cdot|s,a) - p_{\ast}(\cdot|s,a)}_1 > \beta_k \right \}$, otherwise, by Lemma \ref{lm:feasibility}, problem \eqref{eq:objective_cost} would be feasible under $p_k$, and therefore $p_k \in \mathcal{G}$ which contradicts to $p_k \notin \mathcal{G}$. Next, we note that $B_k$ is $\mathcal{F}_{t_k}$-measurable which allows us to use Lemma \ref{lm:posterior_lemma}. Setting $\delta= 1/T$ in Lemma \ref{lm:jaksch_ci} implies that $\prob \left ( \norm{p_k(\cdot|s,a) - p_{\ast}(\cdot|s,a)}_1 > \beta_k \right )$ can be bounded by $\frac{2}{15Tt_k^6}$. Indeed, 
    \begin{align*}
        \prob \left ( \norm{p_k(\cdot|s,a) - p_{\ast}(\cdot|s,a)}_1 > \beta_k \right ) \leq \prob \left ( p_{\ast} \notin B_k \right ) + \prob \left ( p_k \notin B_k \right ) = 2 \prob \left ( p_{\ast} \notin B_k \right ) \leq \frac{2}{15Tt_k^6},
    \end{align*}
    where the last equality follows from Lemma \ref{lm:posterior_lemma} and the last inequality is due to Lemma \ref{lm:jaksch_ci}.

    Finally, we have
    \begin{align*}
        \sum_{k=1}^{K_T} \ee \left [ R_{0,k} \lvert A_1 \right ] \prob \left ( A_1 \right ) \leq \sum_{k=1}^{K_T} \frac{2 (t_{k+1} - t_k)}{15Tt_k^{6}} \leq \frac{2}{15}\sum_{k=1}^{\infty} k^{-6} \leq 1.
    \end{align*}
 
    To bound the term $\sum_{k=1}^{K_T} \ee \left [ R_{0,k} \lvert A_2 \right ] \prob \left ( A_2 \right )$, we rewrite it as
    \begin{align*}
        \sum_{k=1}^{K_T} \ee \left [ R_{0,k} \left.\right| A_2 \right ] \prob \left ( A_2 \right ) & =  \sum_{k=1}^{K_T} \sum_{t=t_k}^{t_{k+1}-1} \ee \left [ (c_0(s_t, a_t) - J^{\pi_{\ast}}(c_0; p_{\ast}) \lvert A_2 \right ] \prob \left ( A_2 \right ) \\
        & \leq \sum_{k=1}^{K_T} \sum_{t=t_k}^{t_{k+1}-1} \prob \left ( A_2 \right ) \leq \sum_{k=1}^{K_T} \sum_{t=t_k}^{t_{k+1}-1} \prob \left ( \exists (s,a):  N_{t_k}(s,a) < \sqrt{T/A} \left.\right| p_k \notin \mathcal{G} \right ), \\ 
    \end{align*}
    where the first inequality holds because $\lvert (c_0(s_t, a_t) - J^{\pi_{\ast}}(c_0; p_{\ast}) \lvert \leq 1$ and the last inequality is by $\prob(A \wedge B) = \prob(A \lvert B) \prob(B)$. 
    Then, by Lemma \ref{lm:exploration_lemma}, we obtain
    \begin{align*}
        &\sum_{k=1}^{K_T} \sum_{t=t_k}^{t_{k+1}-1} \prob \left ( \exists (s,a):  N_{t_k}(s,a) < \sqrt{T/A} \left.\right| p_k \notin \mathcal{G} \right ) \leq 2DS\sqrt{AT} + 1.
    \end{align*}
    
    For the second term of \eqref{eq:regret_decompose}, conditioned on the good event, $\{ p_k \in \mathcal{G} \}$, the sampled CMDP is feasible, and the standard analysis of \citet{TS_MDP_Ouyang_2017} can be applied. Lemma \ref{lm:regret_on_good_event} shows that this term can be bounded by $(D+1) \sqrt{2SAT \log (T)} + 49DS \sqrt{AT \log (AT)}$.

    Putting all bounds together, we obtain the resulting regret bound of:
    \begin{align*}
        BR_+ (T;c_0) & 
        \leq O \left ( DS \sqrt{AT \log (AT)} \right ).
    \end{align*}
    
    \paragraph{Bounding regret of auxiliary cost components.}

    Without loss of generality, fix the cost component $c_i$ and its threshold $\tau_i$ for some $i$ and focus on analyzing the $i$-th component regret. Similarly to the decomposition of the main component, we obtain:
    \begin{align}
    \label{eq:regret_decompose_aux}
        BR_+ (T;c_i) & 
        = \ee \left [ \sum_{t=0}^T \left ( c_i(s_t, a_t) - \tau_i \right )_+ \right ] 
        = \sum_{k=1}^{K_T} \ee \left [ R_{i,k} \right ] 
        \notag
        \\
        & 
        = \sum_{k=1}^{K_T} \ee \left [ R_{i,k} \lvert p_k \notin \mathcal{G} \right ] \prob \left ( p_k \notin \mathcal{G} \right ) 
        + \sum_{k=1}^{K_T} \ee \left [ R_{i,k} \lvert p_k \in \mathcal{G} \right ] \prob \left ( p_k \in \mathcal{G} \right )  
        \notag
    \end{align}
    where $R_{i,k} = \sum_{t=t_k}^{t_{k+1}-1} \left [ c_i(s_t, a_t) -  \tau_i \right ]_+$.

    The first term can be analyzed similarly to the main cost component and bounded by $2DS\sqrt{AT} + 2$. The regret bound of the second term is the same as the regret bound of the analogous term of the main cost component. Its analysis is marginally different and provided in Lemma \ref{lm:regret_on_good_event_aux}. \qed
    
\subsection{Proof of Feasibility lemma (Lemma \ref{lm:feasibility})}
\label{apx:proof_feasibility}

\begin{proof}
    Fix some $i \in \{1, \dots, m\}$. Further, we will omit index $i$ and write $c$ and $\tau$ instead of $c_i$ and $\tau_i$.
    
    With slight abuse of notation, we rewrite the equation \eqref{eq:bellman} in vector form:
    \begin{equation}
        \label{eq:bellman_vector}
        J^{\pi, p, c} + v^{\pi, p} = c_{\pi} +  P_{\pi} v^{\pi, p}.
    \end{equation}
    Above, $J^{\pi, p, c}$, $v^{\pi, p}$, and $c_{\pi}$ are $S$ dimensional vectors of $J^{\pi, p, c}_{s}$, $v^{\pi, p}_{s}$, and $c_{s, \pi(s)}$ with $J_s^{\pi, p, c} = J^{\pi}(s;c, p)$, $v^{\pi, p}_{s} = v^{\pi}(s; p)$, and  $c_{s, \pi(s)} = \sum_{a \in \mathcal{A}} \pi(a|s) c(s,a)$; and $P_{\pi}$ is the transition matrix whose rows formed by the vectors $p_{s, \pi(s)}$, where $p_{s, \pi(s)} = \sum_{a \in \mathcal{A}} \pi(a|s) p(\cdot | s, a)$.

    Let $P_{\bar{\pi}}^k$ be the transition matrix whose rows are formed by the vectors $p^k_{s,\bar{\pi}(s)}$, and $P^{\ast}_{\bar{\pi}}$ be the transition matrix whose rows are formed by the vectors $p^{\ast}_{s,\bar{\pi}(s)}$. Since $N_{t_k}(s,a) \geq \sqrt{T/A}$ for all $(s,a)$, $\norm{p_k(\cdot|s,a) - p_{\ast}(\cdot|s,a)}_1 \leq \sqrt{\frac{14S \log (2A T t_k)}{\max \{1, N_{t_k}(s,a)\}}}$, and the span of the bias function $v^{\bar{\pi}, p_{\ast}}$ is at most $D$ (by Assumption \ref{assum:WASP}), we observe
    \begin{equation*}
        \left ( p_k(\cdot|s,a) - p_{\ast}(\cdot|s,a) \right )^\intercal v^{\bar{\pi}, p_{\ast}} \leq \norm{p_k(\cdot|s,a) - p_{\ast}(\cdot|s,a)}_1 \norm{v^{\bar{\pi}, p_{\ast}} }_\infty \leq \delta D
    \end{equation*}
    where $\delta = \sqrt{\frac{14SA^{1/2} \log (2A T t_k)}{\sqrt{T}}}$. Above implies 
    \begin{equation}
        \label{eq:lim_matrices_diff}
        \left( P_{\bar{\pi}}^k - P^{\ast}_{\bar{\pi}} \right ) v^{\bar{\pi}, p_{\ast}} \leq \delta D \textbf{1}
    \end{equation}
    where $\textbf{1}$ is the vector of all 1s.
    
    Following \cite{NIPS2017_Shipra_OPSRL}, let $(P_{\bar{\pi}}^k)^{\ast}$ denote the limiting matrix for Markov chain with transition matrix $P_{\bar{\pi}}^k$. Observe that $P_{\bar{\pi}}^k$ is aperiodic and irreducible because of Assumption \ref{assum:slater}. This implies that $(P_{\bar{\pi}}^k)^{\ast}$ is of the form $\textbf{1}\boldsymbol{q}^{\intercal}$ where $\boldsymbol{q}$ is the stationary distribution of $P_{\bar{\pi}}^k$ (refer to (A.4) in \cite{Puterman_mdp}). Also, $(P_{\bar{\pi}}^k)^{\ast} P_{\bar{\pi}}^k = (P_{\bar{\pi}}^k)^{\ast}$ and $(P_{\bar{\pi}}^k)^{\ast} \textbf{1} = \textbf{1}$.

    Therefore, the gain of policy $\bar{\pi}$
    \begin{equation*}
        J^{\bar{\pi}, p_k, c} \textbf{1} = (c_{\bar{\pi}}^{\intercal} \boldsymbol{q}) \textbf{1} = (P_{\bar{\pi}}^k)^{\ast} c_{\bar{\pi}}
    \end{equation*}
    
    Now,
    \begin{align*}
        J^{\bar{\pi}, p_k, c} \textbf{1} - J^{\bar{\pi}, p_{\ast}, c} \textbf{1} & = (P_{\bar{\pi}}^k)^{\ast} c_{\bar{\pi}} - J^{\bar{\pi}, p_{\ast}, c} \textbf{1} \\
        & = (P_{\bar{\pi}}^k)^{\ast} c_{\bar{\pi}} - J^{\bar{\pi}, p_{\ast}, c} \left( (P_{\bar{\pi}}^k)^{\ast} \textbf{1} \right ) 
        && \quad \tag{using $(P_{\bar{\pi}}^k)^{\ast} \textbf{1} = \textbf{1}$} \\
        & = (P_{\bar{\pi}}^k)^{\ast} \left( c_{\bar{\pi}} - J^{\bar{\pi}, p_{\ast}, c}  \textbf{1} \right ) \\
        & = (P_{\bar{\pi}}^k)^{\ast} \left( I - P^{\ast}_{\bar{\pi}} \right ) v^{\bar{\pi}, p_{\ast}}
        && \quad \tag{using \eqref{eq:bellman_vector}} \\
        & = (P_{\bar{\pi}}^k)^{\ast} \left( P_{\bar{\pi}}^k - P^{\ast}_{\bar{\pi}} \right ) v^{\bar{\pi}, p_{\ast}}
        && \quad \tag{using $(P_{\bar{\pi}}^k)^{\ast} P_{\bar{\pi}}^k = (P_{\bar{\pi}}^k)^{\ast}$} \\
        & \leq D \delta \textbf{1}.
        && \quad \tag{using \eqref{eq:lim_matrices_diff} and $(P_{\bar{\pi}}^k)^{\ast} \textbf{1} = \textbf{1}$}
    \end{align*}

    Then observing that $D \delta \leq \gamma$, we obtain
    \begin{equation*}
        J^{\bar{\pi}}(c, p_k) - J^{\bar{\pi}}(c, p_{\ast}) \leq D \delta \leq \gamma.
    \end{equation*}
    Using Assumption \ref{assum:slater} and rearranging the terms in the inequality above, it follows
    \begin{equation*}
        J^{\bar{\pi}}(c, p_k) \leq J^{\bar{\pi}}(c, p_{\ast}) + \gamma \leq \tau - \gamma + \gamma \leq \tau.
    \end{equation*}

\end{proof}

\subsection{Proof of Exploration lemma (Lemma \ref{lm:exploration_lemma})}
\label{apx:proof_exploration}
Before providing the proof for the Exploration lemma, we first show that \textsc{PSConRL} requires at most $D$ timesteps to reach a target state when LP \eqref{eq1}-\eqref{eq4} is infeasible.

\begin{lemma}
\label{lm:diam_bound_for_explor_MDP}
    Fix some target state $\Bar{s}$ and its corresponding MDP $M_{\Bar{s}}$ and let $\pi_{\Bar{s}}$ be a solution of Eq. \eqref{eq:opt_bellman}. Then $T_{s \to \Bar{s}}^{\pi_{\Bar{s}}} \leq D$.
\end{lemma}
\begin{proof}
    For simplicity, assume that MDP $M_{\Bar{s}}$ is aperiodic (we will consider the general case later). In such MDP, value iteration is known to converge, and we can find $(J^*, v^*)$ that satisfy Eq. \eqref{eq:opt_bellman} and the corresponding optimal policy $\pi^*_{\Bar{s}}$.
    
    Assume that there exists some policy $\pi$ and state $s$ such that $T^{\pi^*_{\Bar{s}}}_{s \to \Bar{s}} >  T^{\pi}_{s \to \Bar{s}}$. Consider the following policy $\pi^{\prime}$: follow $\pi$ starting from $s$ and wait until $\Bar{s}$ is reached (suppose that this happens in $\tau$ steps), then follow the optimal policy $\pi^*_{\Bar{s}}$. Note that $\tau$ is a random variable and, by definition,
    \begin{equation*}
        \ee[\tau] = T^{\pi^{\prime}}_{s \to \Bar{s}}.
    \end{equation*}

    Let $(J^{\pi^{\prime}}, v^{\pi^{\prime}})$ be the average cost and the bias function of policy $\pi^{\prime}$. First, note that $J^* = J^{\pi^{\prime}}$, since $\pi^{\prime}$ is constructed the way that some policy $\pi$ is utilized for a finite number of steps and the same policy $\pi^*_{\Bar{s}}$ is used in the long term. Next, if $v$ is a bias vector, $v$ plus any constant is also a bias vector. Therefore, without loss of generality, we can apply the following transformation to $v^*$ and $v^{\pi^{\prime}}$: 
    \begin{align}
    \label{eq:bias_transform}
    \begin{split}
        v^* & = v^* - \min_{s \in \mathcal{S}} v^*(s), \\
        v^{\pi^{\prime}} & = v^{\pi^{\prime}} - \min_{s \in \mathcal{S}} v^{\pi^{\prime}}(s). 
    \end{split}
    \end{align}

    Observe that by definition of the cost function $c_{\Bar{s}}$, after transformation $\eqref{eq:bias_transform}$, $v^*(s) = T^{\pi^*_{\Bar{s}}}_{s \to \Bar{s}}$ and $v^{\pi^{\prime}}(s) = T^{\pi^{\prime}}_{s \to \Bar{s}}$. Thus, for state $s$, we obtain
    \begin{equation*}
        J^* + v^*(s) > J^{\pi^{\prime}} + v^{\pi^{\prime}}(s),
    \end{equation*}
    which contradicts the optimality of $(J^*, v^*)$.

    Now, if the MDP $M_{\Bar{s}}$ is periodic, we apply the aperiodicity transformation from \citet{Puterman_mdp} to get a new MDP $\tilde{M}_{\Bar{s}}$: choose $\theta$ satisfying $0 < \theta < 1$ and define $\tilde{\mathcal{S}} = \mathcal{S}$, $\tilde{\mathcal{A}} = \mathcal{A}$, and
    \begin{align*}
        \tilde{c}_{\Bar{s}} & = \theta c_{\Bar{s}}, \\
        \tilde{p}(\cdot | s,a) & = (1 - \theta) \textbf{e}_s +  \theta p(\cdot | s,a).
    \end{align*}
    Note that $\tilde{M}_{\Bar{s}}$ is communicating and aperiodic, and the previous reasoning applies to $\tilde{M}_{\Bar{s}}$. Let $\tilde{J}^{\pi}, \tilde{v}^{\pi}, \tilde{T}^{\pi}_{s \to \Bar{s}}$ denote the quantities associated with $\tilde{M}_{\Bar{s}}$ for some policy $\pi$. Then, by \citet{Puterman_mdp}[Proposition 8.5.8], these are related to the corresponding quantities for $M_{\Bar{s}}$ as follows:
    \begin{align*}
        \tilde{J}^{\pi} & = J^{\pi}, \\
        \tilde{v}^{\pi} & = v^{\pi}, \\
        \tilde{T}^{\pi}_{s \to \Bar{s}} & = \frac{T^{\pi}_{s \to \Bar{s}}}{\theta}.
    \end{align*}
    Using these relations and the fact that we proved the result for $\tilde{M}_{\Bar{s}}$ gives us the result for periodic MDPs. Since $\min_{\pi} T^{\pi}_{s \to \Bar{s}} \leq \max_{s, s^{\prime}}\min_{\pi} T^{\pi}_{s \to s^{\prime}}$, it immediately follows that $T^{\pi^*_{\Bar{s}}}_{s \to \Bar{s}} \leq D$.
\end{proof}

\begin{proof}[Proof of Lemma \ref{lm:exploration_lemma}]
    Let $T_e$ be the first time when every $(s,a)$-pair is visited at least $\sqrt{T/A}$ times given $\{p_k \notin \mathcal{G}\}$, $T_e = \min \{ t: N_t(s,a) \geq \sqrt{T/A} \quad \forall (s,a) \left.\right| p_k \notin \mathcal{G} \}$.

    Since $T_k \leq T_{k-1} + 1$ and $\prob \left ( \exists (s,a):  N_{t}(s,a) < \sqrt{T/A} \left.\right| p_k \notin \mathcal{G} \right )$ is non-increasing in $t$, i.e., $\prob \left ( \exists (s,a):  N_{t}(s,a) < \sqrt{T/A} \left.\right| p_k \notin \mathcal{G} \right ) \leq \prob \left ( \exists (s,a):  N_{t-1}(s,a) < \sqrt{T/A} \left.\right| p_k \notin \mathcal{G} \right )$, for $k>1$, we observe
    \begin{align*}
        \sum_{t=t_k}^{t_{k+1}-1} \prob \left ( \exists (s,a):  N_{t_k}(s,a) < \sqrt{T/A} \left.\right| p_k \notin \mathcal{G} \right ) & \leq \prob \left ( \exists (s,a):  N_{t_k}(s,a) < \sqrt{T/A} \left.\right| p_k \notin \mathcal{G} \right ) \\
        & + \sum_{t=t_{k-1}}^{t_{k}-1} \prob \left ( \exists (s,a):  N_{t}(s,a) < \sqrt{T/A} \left.\right| p_k \notin \mathcal{G} \right ).
    \end{align*}

    Next, by noting that $\prob \left ( \exists (s,a):  N_{t}(s,a) < \sqrt{T/A} \left.\right| p_k \notin \mathcal{G} \right ) = \prob \left ( T_e > t \right )$, we have
    \begin{align*}
        & \sum_{k:t_k \leq T} \ee \left [ \sum_{t=t_k}^{t_{k+1}-1} \mathbb{I} \left \{ \exists (s,a):  N_{t_k}(s,a) < \sqrt{T/A} \right \} \left.\right| p_k \notin \mathcal{G} \right ] = \sum_{k=1}^{K_T} \sum_{t=t_k}^{t_{k+1}-1} \prob \left ( \exists (s,a):  N_{t_k}(s,a) < \sqrt{T/A} \left.\right| p_k \notin \mathcal{G} \right ) \\
        & \leq 1 + \sum_{k=2}^{K_T} \sum_{t=t_k}^{t_{k+1}-1} \prob \left ( \exists (s,a):  N_{t_k}(s,a) < \sqrt{T/A} \left.\right| p_k \notin \mathcal{G} \right ) \\ 
        & \leq 1 + \sum_{k=2}^{K_T} \left [ \prob \left ( \exists (s,a):  N_{t_k}(s,a) < \sqrt{T/A} \left.\right| p_k \notin \mathcal{G} \right ) 
        + \sum_{t=t_{k-1}}^{t_{k}-1} \prob \left ( \exists (s,a):  N_{t}(s,a) < \sqrt{T/A} \left.\right| p_k \notin \mathcal{G} \right ) \right ]\\
        & = 1 + \sum_{k=2}^{K_T} \left [ \prob \left ( T_e > t_k \right ) + \sum_{t=t_{k-1}}^{t_{k}-1} \prob \left ( T_e > t \right ) \right ] = 1 + \sum_{k=2}^{K_T} \prob \left ( T_e > t_k \right ) + \sum_{t=1}^{T} \prob \left ( T_e > t \right ) \leq 1 + 2 \ee [T_e],
    \end{align*}
    where the last inequality follows from the tail sum formula $\ee [T_e] = \sum_{t=0}^{\infty} \prob (T_e > t)$. Finally, by Lemma \ref{lm:diam_bound_for_explor_MDP}, we have $\ee[T_e] \leq DS\sqrt{AT}$, which gives
    \begin{equation*}
        \sum_{k:t_k \leq T} \ee \left [ \sum_{t=t_k}^{t_{k+1}-1} \mathbb{I} \left \{ \exists (s,a):  N_{t_k}(s,a) < \sqrt{T/A} \right \} \left.\right| p_k \notin \mathcal{G} \right ] \leq 2DS\sqrt{AT} + 1.
    \end{equation*}
\end{proof}

\subsection{Auxiliary lemmas}

\begin{lemma}[Regret of the main cost on the good event]
    \label{lm:regret_on_good_event}
    Under Assumption \ref{assum:WASP}, conditioned on the good event $\{ p_k \in \mathcal{G} \}$,
    \begin{equation*}
        \sum_{k=1}^{K_T} \ee \left [ R_{0,k} \lvert p_k \in \mathcal{G} \right ] \prob \left ( p_k \in \mathcal{G} \right ) \leq (D+1) \sqrt{2SAT \log (T)} + 49DS \sqrt{AT \log (AT)}.
    \end{equation*}
\end{lemma}
Most of the analysis here recovers the analysis of \citet{TS_MDP_Ouyang_2017}. Nonetheless, for the sake of clarity, we provide the complete proof of Lemma \ref{lm:regret_on_good_event}.
\begin{proof}
    First, we rewrite equation \eqref{eq:bellman} in terms of the state-action pair \cite{chen_2022_optimisQlearn}:
    \begin{equation}
    \label{eq:bellman_qfunc}
        J^{\pi}(s; c, p) + q^{\pi}(s,a;p) = c(s,a) + \sum_{s^{\prime}} p(s^{\prime}|s,a) v^{\pi}(s^{\prime}; p),
    \end{equation}
    where $v^{\pi}(s; p)$ and $q^{\pi}(s,a;p)$ are connected by $v^{\pi}(s; p) = \sum_a \pi(a|s)  q^{\pi}(s,a;p)$.
    
    Conditioned on the good event $\{ p_k \in \mathcal{G} \}$, every policy $\pi_k$ is the solution of LP \eqref{eq1}-\eqref{eq4}, and we can apply the Bellman equation \eqref{eq:bellman_qfunc} to $c_0(s_t,a_t)$, and decompose $R_{0,k}$ into the following terms.
    \begin{align}
        & \sum_{k=1}^{K_T} \ee \left [ R_{0,k} \lvert p_k \in \mathcal{G} \right ]  \prob \left ( p_k \in \mathcal{G} \right ) \leq \sum_{k=1}^{K_T} \ee \left [ R_{0,k} \lvert p_k \in \mathcal{G} \right ]
        =  
        \sum_{k=1}^{K_T} \ee \left [ \sum_{t=t_k}^{t_{k+1}-1} \big( c_0(s_t,a_t)   -  J^{\pi_{\ast}}(c_0; p_{\ast}) \big) \right ] 
        \notag\\
        & = \sum_{k=1}^{K_T} \ee \left [ \sum_{t=t_k}^{t_{k+1}-1} \left( J^{\pi_k}(c_0; p_k)   -  J^{\pi_{\ast}}(c_0; p_{\ast}) +  q^{\pi_k}(s_t, a_t; p_k) -  \sum_{s' \in \mathcal S}p_k(s'|s_t,a_t)v^{\pi_k}(s',p_k)
        \right)\right]
        \notag\\
        & =
        \underbrace{\sum_k \ee \left [ \sum_t \big( J^{\pi_k}(c_0; p_k) -  J^{\pi_{\ast}}(c_0; p_{\ast}) \big) \right ]}_{R_0} 
        + 
        \underbrace{\sum_k \ee \left [ \sum_t \left ( q^{\pi_k}(s_t, a_t; p_k) - v^{\pi_k}(s_t; p_k) \right ) \right ]}_{R_1}
        \notag\\
        &
        +
        \underbrace{\sum_k \ee \left[\sum_t  \left[  v^{\pi_k}(s_t; p_k) -  v^{\pi_k}(s_{t+1};p_k)
        \right]\right]}_{R_2}
        +
        \underbrace{\sum_k \ee \left[\sum_t  \left[ v^{\pi_k}(s_{t+1};p_k) - \sum_{s'}p_k(s'|s_t,a_t)v^{\pi_k}(s';p_k)
        \right]\right]}_{R_3}.
        \notag
        \label{eq:costregretbound}
    \end{align}
    Now, we note that $R_1 = 0$ as 
    \begin{equation}
    \label{eq:val_fn_q_fn}
        \ee [q^{\pi_k}(s_t, a_t; p_k) - v^{\pi_k}(s_t; p_k)] = \ee [q^{\pi_k}(s_t, a_t; p_k) - \sum_a \pi_k(a|s_t)  q^{\pi}(s_t,a;p_k)] = 0.
    \end{equation}  
    Next, applying lemmas \ref{lm:R0}, \ref{lm:R1}, \ref{lm:R2} to $R_0$, $R_2$, $R_3$, correspondingly, gives us the result.
\end{proof}

\begin{lemma}[Regret of the auxiliary costs on the good event]
    \label{lm:regret_on_good_event_aux}
    Under Assumption \ref{assum:WASP}, conditioned on the good event $\{ p_k \in \mathcal{G} \}$,
    \begin{equation*}
        \sum_{k=1}^{K_T} \ee \left [ R_{i,k} \lvert p_k \in \mathcal{G} \right ] \prob \left ( p_k \in \mathcal{G} \right ) \leq (D+1) \sqrt{2SAT \log (T)} + 49DS \sqrt{AT \log (AT)}.
    \end{equation*}
\end{lemma}

\begin{proof}
    Similarly to Lemma \ref{lm:regret_on_good_event}, conditioned on the good event $\{ p_k \in \mathcal{G} \}$, we can decompose $R_{i,k}$ as follows: 
    \begin{align}
        & \sum_{k=1}^{K_T} \ee \left [ R_{i,k} \lvert p_k \in \mathcal{G} \right ]  \prob \left ( p_k \in \mathcal{G} \right ) \leq \sum_{k=1}^{K_T} \ee \left [ R_{i,k} \lvert p_k \in \mathcal{G} \right ]
        =  
        \sum_{k=1}^{K_T} \ee \left [ \sum_{t=t_k}^{t_{k+1}-1} \big( c_i(s_t,a_t)   -  \tau_i \big) \right ] 
        \notag\\
        & = \sum_{k=1}^{K_T} \ee \left [ \sum_{t=t_k}^{t_{k+1}-1} \left( J^{\pi_k}(c_i; p_k)   -  \tau_i +  q^{\pi_k}(s_t, a_t; p_k) -  \sum_{s' \in \mathcal S}p_k(s'|s_t,a_t)v^{\pi_k}(s',p_k)
        \right)\right]
        \notag\\
        & =
        \underbrace{\sum_k \ee \left [ \sum_t \big( J^{\pi_k}(c_i; p_k) -  \tau_i \right ]}_{R_0} 
        + 
        \underbrace{\sum_k \ee \left [ \sum_t \left ( q^{\pi_k}(s_t, a_t; p_k) - v^{\pi_k}(s_t; p_k) \right ) \right ]}_{R_1}
        \notag\\
        &
        +
        \underbrace{\sum_k \ee \left[\sum_t  \left[  v^{\pi_k}(s_t; p_k) -  v^{\pi_k}(s_{t+1};p_k)
        \right]\right]}_{R_2}
        +
        \underbrace{\sum_k \ee \left[\sum_t  \left[ v^{\pi_k}(s_{t+1};p_k) - \sum_{s'}p_k(s'|s_t,a_t)v^{\pi_k}(s';p_k)
        \right]\right]}_{R_3}.
        \notag
        \label{eq:costregretbound}
    \end{align}
    Now, we note that $\left( J^{\pi_k}(c_i; p_k) - \tau_i \right)$ is negative on the good event $\{ p_k \in \mathcal{G} \}$ for all $k$, and term $R_0$ can be dismissed. $R_1 = 0$ because of \eqref{eq:val_fn_q_fn}, and $R_2$ and $ R_3$ regret terms can be bounded by Lemmas \ref{lm:R1} and \ref{lm:R2} correspondingly.
\end{proof}

\begin{lemma}[Lemma 3 from \cite{TS_MDP_Ouyang_2017}]
For any cost function $c : \mathcal{S} \times \mathcal{A} \xrightarrow{} [0, 1]$, 
\label{lm:R0}
\begin{align*}
    \ee \left [ \sum_{k=1}^{K_T}  \sum_{t=t_k}^{t_{k+1}-1}  \big ( J^{\pi_k}(c;p_k) - J^{\pi_{\ast}}(c;p_{\ast}) \big ) \right ]  \leq K_T \leq \sqrt{2SAT \log(T)}.
\end{align*}
\end{lemma}

\begin{lemma}[Lemma 4 from \cite{TS_MDP_Ouyang_2017}]
\label{lm:R1}
\begin{align*}
    \ee \left [ \sum_{k=1}^{K_T} \sum_{t=t_k}^{t_{k+1}-1} \big(v^{\pi_k}(s_{t}; p_k) - v^{\pi_k}(s_{t+1}; p_k)\big) \right ]  \leq DK_T \leq D \sqrt{2SAT \log(T)}.
\end{align*}
\end{lemma}

\begin{lemma}[Lemma 5 from \cite{TS_MDP_Ouyang_2017}]
\label{lm:R2}
\begin{align*}
    \ee\left [\sum_{k=1}^{K_T}\sum_{t=t_k}^{t_{k+1}-1}  \big( v^{\pi_k}(s_{t+1}; p_k) -\sum_{s' \in \mathcal S}p_k(s'|s_t, a_t)v^{\pi_k}(s';p_k)\big)\right ] \leq 49 D S\sqrt{AT\log(AT)}.
\end{align*}
\end{lemma}

\begin{lemma}[Lemma 17 from \cite{JMLR:v11:jaksch10a}]
\label{lm:jaksch_ci}
For any $t \geq 1$, the probability that the true MDP $M$ is not contained in the set of plausible MDPs $\mathcal{M}(t) = \left \{ (\mathcal{S}, \mathcal{A}, p^{\prime}, \textit{\textbf{c}}, \tau, \rho) : \norm{p^{\prime}(\cdot|s,a) - p_k(\cdot|s,a)}_1 \leq \sqrt{\frac{14S \log (2A t_k / \delta)}{\max \{1, N_{t_k}(s,a)\} }} \right \}$ at time $t$ is at most $\frac{\delta}{15t}$, that is
\begin{align*}
    \prob \left \{ M \notin \mathcal{M}(t) \right \} < \frac{\delta}{15t^6}.
\end{align*}
\end{lemma}

\section{Experimental details}
\label{apx:experiments}
\subsection{Baselines: OFU-based algorithms}
\label{apx:benchmarks}

We use three OFU-based algorithms from the existing literature for comparison: \textsc{C-UCRL} \citep{pmlr-v120-zheng20a}, \textsc{UCRL-CMDP} \citep{Singh_CMDP_2020}, and \textsc{FHA} (Alg. 3) \cite{chen_2022_optimisQlearn}. 
These algorithms rely on the knowledge of different CMDP components, e.g., \textsc{UCRL-CMDP} relies on knowledge of rewards $r$, whereas \textsc{C-UCRL} uses the knowledge of transitions $p$. To enable fair comparison, all algorithms were extended to the unknown reward/costs and unknown probability transitions setting. Specifically, we assume that each algorithm knows only the states space $\mathcal{S}$ and the action space $\mathcal{A}$, substituting the unknown elements with their empirical estimates:

\begin{equation}
    \label{mean_rewards}
    \Bar{r}_t(s,a) = \frac{\sum_{j=1}^{t-1} \mathbb{I} \{ s_t=s, a_t=a \} r_t}{N_t(s,a) \lor 1}, \quad \forall s \in \mathcal{S}, a \in \mathcal{A},
\end{equation}

\begin{equation}
    \label{mean_costs}
    \Bar{c}_{i,t}(s,a) = \frac{\sum_{j=1}^{t-1} \mathbb{I} \{ s_t=s, a_t=a \} c_{i,t}}{N_t(s,a) \lor 1}, \quad \forall s \in \mathcal{S}, a \in \mathcal{A}, \quad i = 1 \dots, m,
\end{equation}

\begin{equation}
    \label{mean_transitions}
    \Bar{p}_t(s,a,s') = \frac{N_t(s,a,s')}{N_t(s,a) \lor 1}, \quad \forall s, s' \in \mathcal{S}, a \in \mathcal{A}.
\end{equation}
where $r$ is the reward function (inverse main cost $c_0$) and $N_t(s,a)$ and $N_t(s,a,s')$ denote the number of visits to $(s,a)$ and $(s,a,s')$ respectively.

Further, we provide algorithmic-specific details separately for each baseline:

\begin{enumerate}
    
    \item \textsc{C-UCRL} follows a principle of “optimism in the face of reward uncertainty; pessimism in the face of cost uncertainty.” This algorithm, 
    which was developed in \cite{pmlr-v120-zheng20a}, considers conservative (safe) exploration by overestimating both rewards and costs:
    
    \begin{align*}
        \hat{r}_t(s,a) =\Bar{r}_t(s,a) + b_t(s,a)
            \quad\mathrm{and}\quad
        \hat{c}_t(s,a) =\Bar{r}_t(s,a) + b_t(s,a).
    \end{align*}
    
    \textsc{C-UCRL} proceeds in episodes of linearly increasing number of rounds $kh$, where $k$ is the episode index and $h$ is the fixed duration given as an input. In each epoch, the random policy \footnote{Original algorithm utilizes a safe baseline during the first $h$ rounds in each epoch, which is assumed to be known. However, to make the comparison as fair as possible, we assume that a random policy is applied instead.} is executed for $h$ steps for additional exploration, and then policy $\pi_k$ is applied for $(k-1)h$ number of steps, making $kh$ the total duration of episode $k$. 
    
    \item Unlike the previous algorithm, where uncertainty was taken into account by enhancing rewards and costs, \textsc{UCRL-CMDP} \cite{Singh_CMDP_2020} constructs confidence set $\mathcal{C}_t$ over $\Bar{p}_t$:
    
    \begin{equation*}
        \mathcal{C}_t = \left \{ p: |p(s,a,s') - \Bar{p}_t(s,a,s')| \leq b_t(s,a) \quad \forall (s,a) \right \}.
    \end{equation*}
    
    \textsc{UCRL-CMDP} algorithm proceeds in episodes of fixed duration of $\ceil*{T^{\alpha}}$, where $\alpha$ is an input of the algorithm. At the beginning of each round, the agent solves the following constrained optimization problem in which the decision variables are (i) Occupation measure $\mu(s,a)$, and (ii) “Candidate” transition $p'$:

    \begin{align}
        \max_{\mu, p' \in \mathcal{C}_t} \sum_{s,a} \mu(s,a) r(s,a), \label{eq1'}\\
        \mathrm{s.t.}\quad \sum_{s,a} \mu(s,a) c_i(s,a) \leq \tau_i,\, \quad i=1,\dots,m, \\
        \sum_a \mu(s,a) = \sum_{s', a} \mu(s', a) p'(s',a,s), \quad \forall s \in \mathcal{S}, \label{eq3'} \\
        \mu(s,a) \geq 0, \quad \forall (s,a) \in \mathcal{S} \times \mathcal{A}, \quad \sum_{s,a} \mu(s,a) = 1, \label{eq4'}
    \end{align}
    
    Note that program (\ref{eq1'})-(\ref{eq4'}) is not linear anymore as $\mu(s', a)$ is being multiplied by $p'(s',a,s)$ in equation (\ref{eq3'}). This is a serious drawback of \textsc{UCRL-CMDP} algorithm because, as we show in the experiments, program (\ref{eq1'})-(\ref{eq4'}) becomes computationally inefficient for even moderate problems.

    \item \textsc{FHA} (Finite Horizon Approximation for CMDP) divides the $T$ timesteps into $K$ rounds and treats each episode as an episodic finite-horizon CMDP. Fix some episode $k$. Through the lens of occupancy measure that is defined on $S\times A\times H \times S$ space (where $H$ is the length of the episode), \textsc{FHA} optimizes the following linear program:

    \begin{align}
        \max_{\mu} \sum_h \sum_{s,a} r(s,a) \sum_{s'}\mu(s,a,h,s'), \label{eq1''}\\
        s.t. \sum_h \sum_{s,a} c_i(s,a) \sum_{s'}\mu(s,a,h,s') \leq H\tau_i + sp(p_{\ast}), \\
        P_{\mu} \in \left \{ p: |p(s,a,s') - \Bar{p}_k(s,a,s')| \leq b_k(s,a) \quad \forall (s,a) \right \}, \label{eq4''},
    \end{align}
    where $P_{\mu}(s, a, s') = \frac{\mu(s, a, h, s')}{\sum_{s'}\mu(s, a, h, s')} \quad \forall h=1,\dots H$.

    Although program  (\ref{eq1''})-(\ref{eq4''}) is linear, we emphasize that this algorithm requires finding an optimal occupancy measure for each $H$ and each $K$, resulting in $O(S^2AT)$ decision variables. As we mentioned in the experiments, this is prohibitive even for moderate-sized CMDPs. 

\end{enumerate}

\subsection{Environments}
\label{subsec:envs}

To demonstrate the performance of the algorithms, we consider three gridworld environments in our analysis. There are four actions possible in each state, $\mathcal{A} = \{up, down, right, left\}$, which cause the corresponding state transitions, except that actions that would take the agent to the wall leave the state unchanged. Due to the stochastic environment, transitions are stochastic (i.e., even if the agent's action is to go \textit{up}, the environment can send the agent with a small probability \textit{left}). Typically, the gridworld is an episodic task where the agent receives cost 1 (equivalently reward -1) on all transitions until the terminal state is reached. We reduce the episodic setting to the infinite-horizon setting by connecting terminal states to the initial state. Since there is no terminal state in the infinite-horizon setting, we call it the goal state instead. Thus, every time the agent reaches the goal, it receives a cost of 0 (or reward of 0), and every action from the goal state sends the agent to the initial state. We introduce constraints by considering the following specifications of a gridworld environment: Marsrover and Box environments.

\paragraph{Marsrover.} This environment was used in \cite{tessler2018reward, pmlr-v120-zheng20a, NEURIPS2020_Brantley}. The agent must move from the initial position to the goal avoiding risky states. Figure (\ref{marsrover_envs}) illustrates the CMDP structure: the initial position is light green, the goal is dark green, the walls are gray, and risky states are purple. "In the Mars exploration problem, those darker states are the states with a large slope that the agents want to avoid. The constraint we enforce is the upper bound of the per-step probability of stepping into those states with large slope -- i.e., the more risky or potentially unsafe states to explore" \citep{pmlr-v120-zheng20a}.  Each time the agent appears in a purple state incurs an auxiliary cost of 1. Other states incur no auxiliary costs.

\begin{figure}[t]
  \centering
  \subfigure[Marsrover 4x4]{\includegraphics[width=0.25\textwidth]{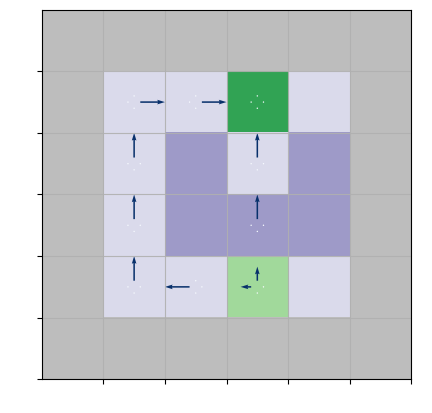}
    \label{4x4_marsrover}}
  \subfigure[Marsrover 8x8]{\includegraphics[width=0.25\textwidth]{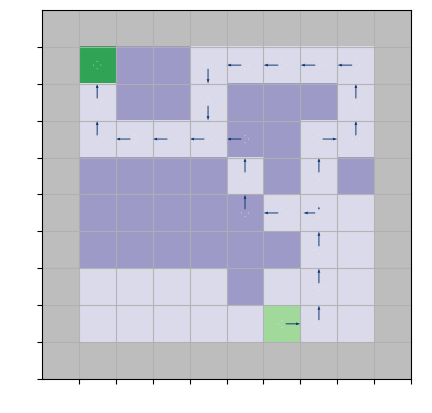}
    \label{8x8_marsrover}}
  \caption{Marsrover gridworlds. The initial position is light green, the goal is dark green, the walls are gray, and risky states are purple. Figure \ref{4x4_marsrover} illustrates 4x4 Marsrover environment. Figure \ref{8x8_marsrover} illustrates 8x8 Marsrover environment. In both cases, the agent's task is to get from the initial state to the goal state, and the optimal policy combines with some probabilities fast and safe ways, which are indicated by arrows on the pictures.}
\label{marsrover_envs}
\vspace*{-1.2\baselineskip}
\end{figure}

Without constraints, the optimal policy is obviously to always go \textit{up} from the initial state. However, with constraints, the optimal policy is a randomized policy that goes \textit{left} and \textit{up} with some probabilities, as illustrated in Figure \ref{4x4_marsrover}. In experiments, we consider two marsrover gridworlds: 4x4, as shown in Figure \ref{4x4_marsrover}, and 8x8, depicted in Figure \ref{8x8_marsrover}.

\paragraph{Box.} Another conceptually different specification of a gridworld is Box environment from \cite{Leike_aisafegrid_2017}. Unlike the Marsrover example, there are no static risky states; instead, there is an obstacle, a box, which is only "pushable" (see Figure \ref{box_main}). Moving onto the blue tile (the box) pushes the box one tile into the same direction if that tile is empty; otherwise, the move fails as if the tile were a wall. The main idea of Box environment is "to minimize effects unrelated to their main objectives, especially those that are irreversible or difficult to reverse" \citep{Leike_aisafegrid_2017}. If the agent takes the fast way (i.e., goes down from its initial state; see Figure \ref{box_down}) and pushes the box into the corner, the agent will never be able to get it back, and the initial configuration would be irreversible. In contrast, if the agent chooses the safe way (i.e., approaches the box from the left side), it pushes the box to the reversible state (see Figure \ref{box_left}). This example illustrates situations of performing the task without breaking a vase in its path, scratching the furniture, bumping into humans, etc.

Each action incurs an auxiliary cost of 1 if the box is in a corner (cells adjacent to at least two walls) and no auxiliary costs otherwise. Similarly to the Marsrover example, without safety constraints, the optimal policy is to take a fast way (go down from the initial state). However, with constraints, the optimal policy is a randomized policy that goes down and left from the initial state.
\begin{figure}[t]
  \centering
  \subfigure[Box (Main)]{\includegraphics[width=0.27\textwidth]{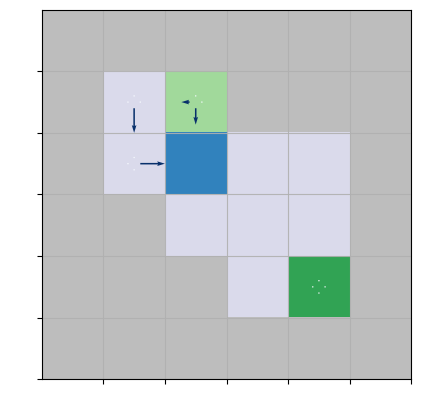}
    \label{box_main}}
  \subfigure[Box (Safe)]{\includegraphics[width=0.27\textwidth]{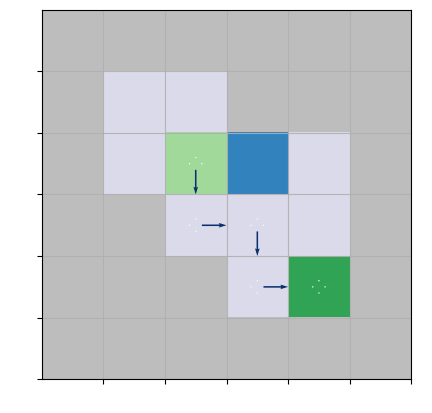}
    \label{box_left}}
  \subfigure[Box (Fast)]{\includegraphics[width=0.27\textwidth]{pics/envs/box_left.png}
    \label{box_down}}
  \caption{Box gridworld. The initial position is light green, the goal is dark green, the walls are gray, and risky states are purple. Figure \ref{box_main} illustrates the initial configuration. The agent's task is to get from the initial state to the goal state, and the optimal policy combines with some probabilities fast and safe ways, which are indicated by arrows on the pictures. Figure \ref{box_left}-\ref{box_down} illustrates safe and fast ways.}
\label{box_envs}
\vspace*{-1.2\baselineskip}
\end{figure}

\subsection{Simulation results}

Figure \ref{fig:sum_results_app} shows the reward (inverse main cost) and average consumption (auxiliary cost) behavior of \textsc{PSConRL}, \textsc{C-UCRL}, \textsc{UCRL-CMDP}, and \textsc{FHA} (Alg. 3) illustrating how the regret from Figure \ref{fig:sum_results} is accumulated. The top row shows the reward performance. The bottom row presents the average consumption of the auxiliary cost. 

Taking a closer look at Marsrover environments (left and middle columns), we see that all algorithms converge to the optimal solution (top row), and their average consumption (middle row) satisfies the constraints in the long run. In the Box example (right column), we see that \textsc{C-UCRL} is stuck with the suboptimal solution. The algorithm exploits safe policy once it is learned, which corresponds to the near-linear regret behavior in Figure \ref{fig:sum_results}. Alternatively, \textsc{PSConRL} converges to the optimal solution relatively quickly (middle and bottom graphs).

\begin{figure}[ht]
    \centering
    \includegraphics[width=0.94\textwidth]{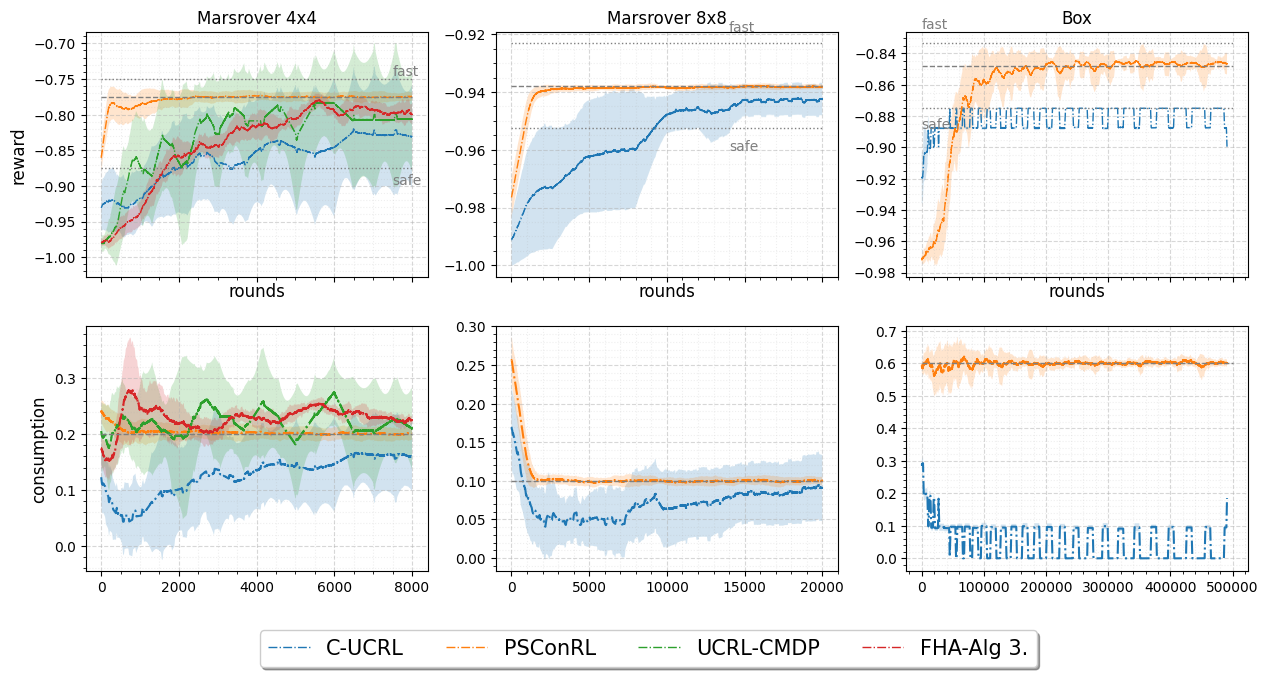}
    \caption{
    \textbf{(Top row)} shows the average reward (inverse average main cost); the dashed line shows the optimal behavior, and the dotted lines depict the reward level of safe and fast policies. \textbf{(Bottom row)} shows the average consumption of the auxiliary cost; the constraint thresholds are 0.2 for Marsrover 4x4, 0.1 for Marsrover 8x8, and 0.6 for Box.  Results are averaged over 100 runs for Marsrover 4x4 and over 30 runs for Marsrover 8x8 and Box.}
    \label{fig:sum_results_app}
    
\end{figure}


\end{document}